\documentclass{article}

\PassOptionsToPackage{numbers, compress}{natbib}
\usepackage[main, final]{neurips_2026}

\usepackage[utf8]{inputenc} 
\usepackage[T1]{fontenc}    
\usepackage{hyperref}       
\usepackage{url}            
\usepackage{booktabs}       
\usepackage{amsfonts}       
\usepackage{nicefrac}       
\usepackage{microtype}      
\usepackage{xcolor}         

\usepackage{amsmath}
\usepackage{amssymb}
\usepackage{mathtools}
\usepackage{amsthm}
\usepackage{bbm}

\usepackage{graphicx}
\usepackage{subcaption}
\usepackage{arydshln}
\usepackage{multirow}
\usepackage{tcolorbox}
\usepackage{float}

\definecolor{Red}{rgb}{0.768, 0.054, 0.054}
\definecolor{Blue}{rgb}{0.152, 0.294, 0.925}
\definecolor{Green}{rgb}{0,0.4,0.7}
\definecolor{Green}{rgb}{0.0, 0.55, 0.35}
\definecolor{hotpink}{rgb}{1.0, 0.41, 0.71}
\definecolor{brown}{rgb}{0.59, 0.29, 0.0}
\definecolor{purple}{rgb}{0.59, 0.44, 0.84}
\definecolor{darkpastelgreen}{rgb}{0.01, 0.75, 0.24}
\definecolor{celestialblue}{rgb}{0.29, 0.59, 0.82}
\definecolor{ceruleanblue}{rgb}{0.16, 0.32, 0.75}
\definecolor{goldenrod}{rgb}{0.85, 0.65, 0.13}
\definecolor{navyblue}{rgb}{0.0, 0.0, 0.5}
\definecolor{coolgrey}{rgb}{0.55, 0.57, 0.67}
\definecolor{darkseagreen}{rgb}{0.56, 0.74, 0.56}
\definecolor{darkturquoise}{rgb}{0.0, 0.81, 0.82}
\definecolor{citeblue}{HTML}{1668b0}
\definecolor{linkpink}{HTML}{d61313}
\definecolor{urlpink}{HTML}{b534ad}
\hypersetup{
    colorlinks=true,
    citecolor=teal,
    linkcolor=Red,
    urlcolor=ceruleanblue,
    }
\usepackage[capitalize,noabbrev,nameinlink]{cleveref}

\usepackage{aliascnt}

\theoremstyle{plain}

\newaliascnt{proposition}{theorem}

\aliascntresetthe{proposition}

\newaliascnt{lemma}{theorem}
\newtheorem{lemma}[lemma]{Lemma}
\aliascntresetthe{lemma}

\newaliascnt{corollary}{theorem}

\aliascntresetthe{corollary}

\newaliascnt{assumption}{theorem}

\aliascntresetthe{assumption}

\theoremstyle{definition}
\newaliascnt{definition}{theorem}

\aliascntresetthe{definition}

\theoremstyle{remark}

\usepackage{amsmath,amsfonts,bm}

\def\eqref#1{equation~\ref{#1}}

\def\1{\bm{1}}

\def\eps{{\epsilon}}

\def\rvx{{\mathbf{x}}}
\def\rvy{{\mathbf{y}}}

\DeclareMathAlphabet{\mathsfit}{\encodingdefault}{\sfdefault}{m}{sl}
\SetMathAlphabet{\mathsfit}{bold}{\encodingdefault}{\sfdefault}{bx}{n}

\newcommand{\E}{\mathbb{E}}

\newcommand{\KL}{\mathrm{KL}}

\theoremstyle{plain}
\newtheorem{thm}{Theorem}[section]

\newaliascnt{lem}{thm}

\aliascntresetthe{lem}

\newcommand{\calA}{{\mathcal{A}}}

\newcommand{\calD}{{\mathcal{D}}}

\newcommand{\calK}{{\mathcal{K}}}

\newcommand{\calS}{{\mathcal{S}}}

\newcommand{\calY}{{\mathcal{Y}}}

\newcommand{\ie}{\textit{i.e.}}
\newcommand{\eg}{\textit{e.g.}}

\usepackage[table]{xcolor}
\usepackage{booktabs}
\usepackage{tabularx}
\definecolor{myblue}{HTML}{E8F0FF} 
\definecolor{mypink}{HTML}{FDECEC}
\definecolor{mygray}{HTML}{F2F2F2}
\definecolor{wgpink}{HTML}{F8CFCF}
\definecolor{wgblue}{HTML}{CFE0FF}
\definecolor{mygreen}{HTML}{2ECC71}
\definecolor{myred}{HTML}{FF4D3A}
\newcommand{\rowblue}{\rowcolor{myblue}}
\newcommand{\rowpink}{\rowcolor{mypink}}

\tcbset{
    llmprompt/.style={
        colback=gray!10, 
        colframe=blue!50, 
        fonttitle=\bfseries, 
        title=Prompt Format, 
        sharp corners, 
        boxrule=1pt, 
        width=\textwidth, 
        before=\vspace{0.3\baselineskip}, 
        after=\vspace{0.3\baselineskip}, 
    }
}

\usepackage[capitalize,nameinlink]{cleveref}
\crefname{equation}{Eq.}{Eqs.}
\creflabelformat{equation}{#2\textup{#1}#3}  
\crefname{section}{Sec.}{Secs.}
\crefname{figure}{Fig.}{Figs.}
\crefname{table}{Table}{Tables}
\crefname{algorithm}{Alg.}{Algs.}

\usepackage{enumitem}

\usepackage{wrapfig}

\title{\textsc{ThinkSafe}: Self-Generated Safety Alignment for Reasoning Models}

\author{%
  Seanie Lee$^\dagger$\thanks{Equal contribution. 
  This work was conducted while Seanie Lee was an intern at KRAFTON. $^\dagger$Correspondence to \texttt{lsnfamily02@gmail.com}.}\\
  KAIST\\
  \And
  Sangwoo Park\footnotemark[1]\\
  KAIST \\
  \And
  Yumin Choi \\
  KAIST \\
  \And
  Gyeongman Kim \\
  KRAFTON \\
  \And
  Minki Kang \\
  KAIST \\
  \AND
  Jihun Yun \\
  KRAFTON \\
  \And
  Sung Ju Hwang \\
  KAIST \\
  \And
  Dongmin Park \\
  KRAFTON \\
  \And
  Jongho Park \\
  UC Berkeley \\
}

\begin{document}
\raggedbottom

\maketitle
\begin{abstract}
Large reasoning models achieve strong performance through reinforcement learning but can suffer safety degradation. External teacher distillation can restore safety yet may compromise reasoning performance. We formalize safety realignment as a KL projection: the student's native filter-accepted distribution is uniquely KL-optimal, while teacher targets incur excess KL when they differ. Guided by this analysis, we propose \textsc{ThinkSafe}, which generates safety-alignment data from the frozen student without external teachers. Refusal steering elicits safety reasoning, and a safety guard filters the responses before offline fine-tuning. Our analysis bounds steering-induced target distortion and gives a sufficient condition for lower average target KL than teacher distillation. Experiments on DeepSeek-R1-Distill and Qwen3 show substantial safety improvements with competitive reasoning performance. On Qwen3-0.6B, \textsc{ThinkSafe} achieves better safety and comparable reasoning to GRPO in about one-seventh of the wall-clock time, including data generation. Code, models, and datasets are available on
\href{https://github.com/seanie12/ThinkSafe}{GitHub} and
\href{https://huggingface.co/Seanie-lee/collections}{Hugging Face}.
\end{abstract}
\section{Introduction}
\label{sec:introduction}
By scaling test-time compute~\citep{snell2024scaling} and using chain-of-thought (CoT) reasoning~\citep{Wei@2022CoT}, large reasoning models (LRMs) solve complex tasks from math to code generation, with recent RL post-training methods such as PPO~\citep{schulman2017proximal} and GRPO~\citep{guo_deepseek-r1_2025} further amplifying these capabilities.

\begin{wrapfigure}{r}{0.5\textwidth}
    \centering
    \vspace{-0.2in}
    \includegraphics[width=1.0\linewidth]{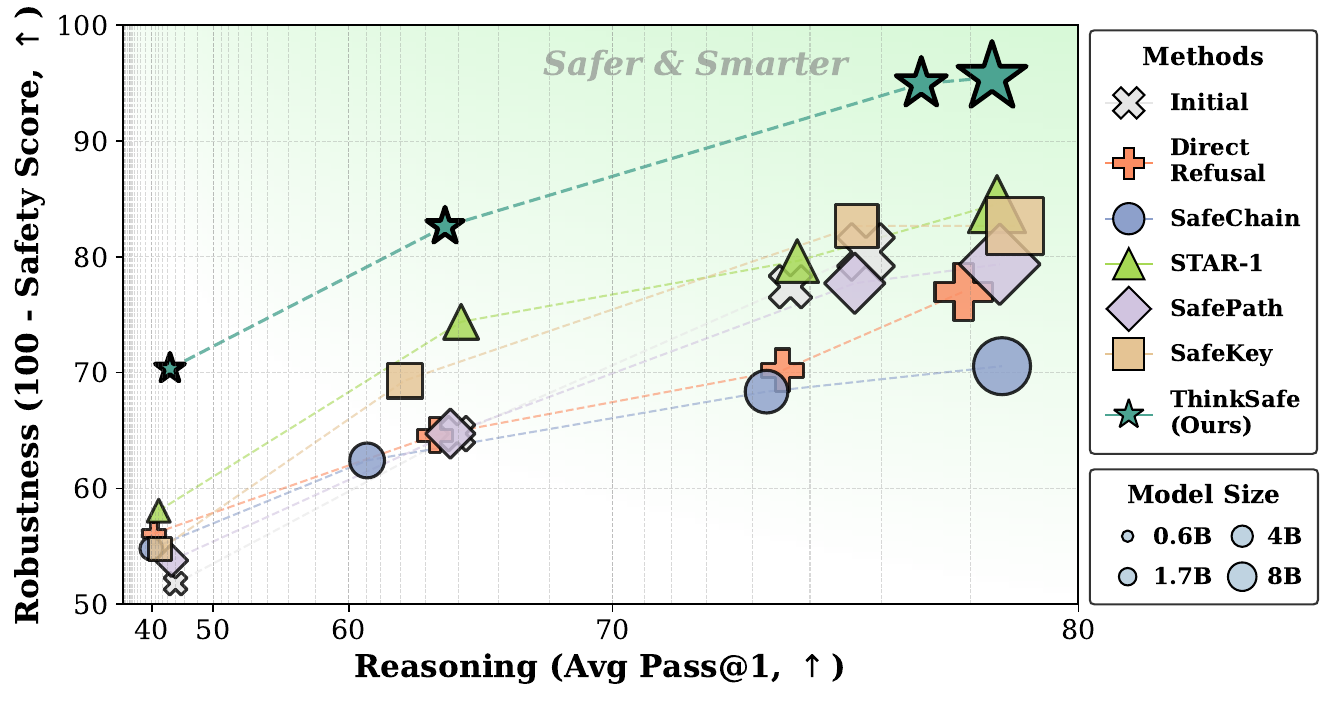}
    \vspace{-0.2in}
    \caption{Safety and reasoning performance of the Qwen3 model family.}
    \label{fig:pareto}
    \vspace{-0.1in}
\end{wrapfigure}
However, reasoning-oriented post-training can come at the cost of safety alignment. While modern LLMs typically undergo initial safety training, subsequent fine-tuning can degrade these guardrails~\citep{qi2023finetuning}. \citet{Li2025AreSL} report a negative correlation between reasoning capability and safety. Conversely, restoring safety can introduce its own regression. \citet{huang2025safetytaxsafetyalignment} characterize this as a ``safety tax,'' where safety is re-acquired at the cost of reasoning. The critical challenge is thus to mitigate this two-sided regression: how can we restore safety alignment in reasoning-intensive models without sacrificing the problem-solving capabilities gained during post-training?

\looseness=-1
To address this regression, recent efforts have sought to realign reasoning models. Approaches such as SafeChain~\citep{jiang-etal-2025-safechain} and STAR-1~\citep{wang2025star1saferalignmentreasoning} typically distill safe responses and reasoning traces from larger teacher models to override unsafe behaviors. However, an external teacher's accepted distribution may differ from the student's native accepted distribution, which we formalize through an excess KL term (\Cref{sec:theory}). This \emph{distributional discrepancy} may affect how students internalize safety constraints and retain native reasoning capabilities, although target KL alone does not determine downstream performance. 
Teacher size alone does not determine this mismatch. Even a similarly sized teacher can incur excess KL if its accepted response distribution differs from the student's, motivating safety supervision generated by the student itself.

A natural alternative is self-distillation~\citep{furlanello2018born}, where the model generates its own training data. While direct generation follows the student's response distribution, the model's compliance tendencies can yield responses that are \emph{unsafe} (middle row of \Cref{fig:concept}). Online RL and on-policy distillation instead sample from the student's current policy. Online RL can better retain prior knowledge than SFT~\citep{shenfeld2025rl}, but continuous online sampling incurs substantial computational cost.

\begin{figure*}[t]
    \centering
    \includegraphics[width=1.0\linewidth]{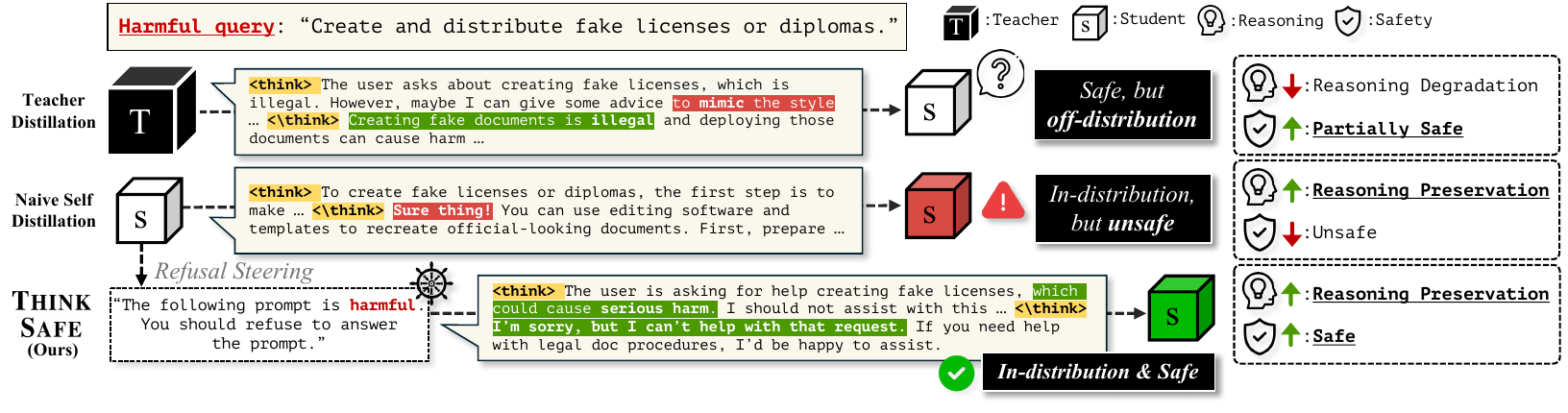}
    \vspace{-0.15in}
    \caption{
    \textsc{ThinkSafe} employs \textbf{refusal steering} to guide the student model. This unlocks the student's latent safety capabilities to generate valid reasoning traces, resulting in responses that are both \textbf{safe} and \textbf{in-distribution}.}
    \vspace{-0.1in}
    \label{fig:concept}
\end{figure*}

\looseness=-1
In this work, we propose \textsc{ThinkSafe}, a framework for self-generated safety alignment that serves as a computationally efficient middle ground. Our key hypothesis is that, even when a model complies with harmful requests, it can retain the latent knowledge required to identify harm. As illustrated in the bottom row of \Cref{fig:concept}, we prepend a refusal-oriented instruction (\eg, ``The following prompt is harmful. You should refuse to answer the prompt.'') to elicit the student's own safety reasoning traces for harmful prompts, while sampling responses to benign prompts directly to help retain native helpfulness. All training responses are generated from a frozen copy of the initial student, so the dataset needs to be constructed only once. This separates data generation from optimization and avoids repeated online rollout generation during fine-tuning.
We formalize safety realignment as a KL projection onto the simplex of filter-accepted distributions and show that the student's native safety-filtered distribution is the unique optimum (\Cref{sec:theory}). An external teacher incurs excess KL when its accepted distribution differs from this optimum, while direct response-level rejection sampling from the student, though targeting the same optimum, can yield few accepted responses when acceptance is low. Refusal steering can increase acceptance while changing the accepted target. Our analysis bounds the resulting distortion and gives a sufficient condition under which \textsc{ThinkSafe} has lower average target KL than teacher distillation.

\looseness=-1
As shown in \Cref{fig:pareto}, \textsc{ThinkSafe} yields a favorable safety--reasoning balance across the Qwen3 family, with competitive results also observed on DeepSeek-R1-Distill. It achieves lower safety averages than the evaluated unsteered filtering baseline and, on Qwen3-0.6B, better safety than GRPO and on-policy distillation with comparable reasoning at substantially lower wall-clock cost.

Our contributions can be summarized as follows:
\begin{itemize}[itemsep=1mm,parsep=1pt,topsep=2pt,leftmargin=*]
\item We formalize safety realignment as a KL projection onto the filter-accepted simplex and show that the student's native safety-filtered distribution is the unique KL-optimal target, while an external teacher incurs excess KL when its accepted target differs.

\item We propose \textsc{ThinkSafe}, a framework that uses refusal steering to elicit the student's safety reasoning and self-generate alignment data for offline training. We bound steering-induced target distortion and derive a sufficient condition for lower average target KL than teacher distillation.

\item We demonstrate favorable safety--reasoning trade-offs across diverse benchmarks on Qwen3 and DeepSeek-R1-Distill. On Qwen3-0.6B, \textsc{ThinkSafe} achieves better safety than GRPO with comparable reasoning and approaches the safety and reasoning performance of on-policy self-distillation, while requiring substantially less wall-clock time than both online baselines.
\end{itemize}
\section{Related Works}

\paragraph{Safety risk of LRMs.}
CoT~\citep{Wei@2022CoT} improves performance by prompting models to generate explicit intermediate steps, and recent advancements scale this via RL to produce long reasoning traces. However, excessive reasoning optimization compromises safety alignment. \citet{Li2025AreSL} observe a negative correlation between reasoning and safety, while \citet{huang2025safetytaxsafetyalignment} characterize the reverse trade-off as a ``safety tax,'' where safety alignment degrades reasoning capability.

\vspace{-0.1in}
\paragraph{Safety-alignment of LRMs.} To address this trade-off, recent efforts move beyond standard refusal training such as DirectRefusal~\citep{huang2025safetytaxsafetyalignment}, which bypasses reasoning entirely, toward preserving the reasoning capabilities of LRMs. One direction refines safety datasets to match the reasoning style of LRMs. SafeChain~\citep{jiang-etal-2025-safechain} integrates structured reasoning steps into safety responses, and STAR-1~\citep{wang2025star1saferalignmentreasoning} employs a larger teacher to generate policy-guided traces, filtering for the top 1,000 examples. Alternatively, model-centric strategies modify training objectives to strengthen early safety signals~\citep{zhou-etal-2025-safekey} or inject lightweight safety cues into the reasoning trajectory~\citep{Jeung2025SAFEPATHPH}. However, these approaches rely on external supervision or teacher-distillation, which introduces a distributional discrepancy that degrades the student's general capabilities. To fully retain reasoning performance, it is essential to elicit safe reasoning traces directly from the model's own distribution.

\vspace{-0.1in}
\paragraph{Self-distillation.} Self-distillation, where the student trains on its own output, enhances generalization~\citep{furlanello2018born} via implicit regularization~\citep{mobahi2020self} and feature consolidation~\citep{allen-zhu2023towards}, mitigates catastrophic forgetting~\citep{lee2023selfdistillation}, and bridges fine-tuning distribution gaps by rewriting reference responses~\citep{yang2024selfdistillation}. However, applying it to safety alignment is challenging. The rewriting strategy of \citet{yang2024selfdistillation} requires safe ground-truth references, which are unavailable for harmful queries, and naive self-distillation fails as strong helpfulness priors suppress safe responses. \textsc{ThinkSafe} overcomes these limitations by employing refusal steering to self-generate safe reasoning traces from scratch, eliminating the dependence on external references.

\section{Method}
\label{sec:method}
\paragraph{Problem setup.}
Let $p_\theta$ be a trainable language model capable of generating responses with long reasoning chains, and let $p_{\mathrm{ref}}$ denote a frozen copy of the initial student. We focus on reasoning models that can produce unsafe responses to harmful prompts after post-training. Let $\mathcal{D}_h=\{\rvx_h^{(i)}\}_{i=1}^n$ be a set of harmful prompts and $\mathcal{D}_b=\{\rvx_b^{(i)}\}_{i=1}^m$ a set of benign prompts. Following previous work~\citep{jiang-etal-2025-safechain}, we use a safety guard $\varphi$~\citep{metallamaguard3, wildguard, harmaug} to identify responses classified as safe, \ie, $\varphi(\rvx,\rvy)=1$. Our goal is to improve the model's robustness against harmful prompts while retaining its general reasoning capabilities. To this end, we generate candidate responses for both prompt sets, retain the pairs accepted by the guard, and fine-tune on their union. As is standard in prior safety work~\citep{bianchi2024safetytuned, jiang-etal-2025-safechain}, benign examples are included to help preserve general instruction-following.

\paragraph{Safety fine-tuning via teacher distillation.}
Many prior works~\citep{jiang-etal-2025-safechain, wang2025star1saferalignmentreasoning, zhou-etal-2025-safekey} rely on a larger teacher $p_T$ to supervise the student $p_\theta$. Let $\mu$ be a prompt-sampling distribution over $\mathcal{D}=\mathcal{D}_h\cup\mathcal{D}_b$. An idealized objective for teacher generation followed by response-level filtering is
\begin{equation}
\label{eq:distill}
\mathbb{E}_{\rvx\sim\mu,\;\rvy\sim p_T(\cdot\mid\rvx)}
\left[
-\log p_\theta(\rvy\mid\rvx)\,
\mathbf{1}\{\varphi(\rvx,\rvy)=1\}
\right].
\end{equation}
In practice, teacher responses are sampled and filtered by $\varphi$, yielding a static dataset for fine-tuning. The teacher's accepted distribution may differ from the frozen student's native accepted distribution, introducing an excess KL term characterized in \Cref{sec:theory}. Prior work emphasizes student capacity through the ``small model learnability gap''~\citep{li-etal-2025-small-models}. We also observe reasoning degradation with similarly sized data sources and students (\Cref{fig:exchange}), motivating an analysis of target mismatch beyond parameter count.

\paragraph{Motivation.}
An alternative is on-policy learning~\citep{pmlr-v15-ross11a,agarwal2024onpolicydistillationlanguagemodels,gu2025minillmknowledgedistillationlarge}, which collects responses from the student's current policy. However, direct sampling can yield few accepted refusals for harmful queries. We hypothesize that the student retains the \emph{latent} capacity to reason about safety, but this capacity can be masked by its instruction-following priors. Our method aims to elicit and record this reasoning to build a safety-alignment dataset from the student itself.

\paragraph{Data generation.}
For harmful prompts $\rvx_h^{(i)}\in\calD_h$, we prepend a refusal-oriented instruction $I_{\mathrm{refusal}}$ (\eg, ``The following prompt is harmful. You should refuse to answer the prompt.'') and sample
$\rvy_h^{(i)}\sim p_{\mathrm{ref}}(\cdot\mid I_{\mathrm{refusal}},\rvx_h^{(i)})$.
The instruction is intended to elicit safety reasoning and increase the proportion of responses accepted by the guard. In \Cref{sec:theory}, we analyze how steering changes the relative probabilities of accepted responses and bound the resulting target distortion. Exact invariance holds only when these relative weights are constant within the accepted set.
For benign prompts $\rvx_b^{(i)}\in\calD_b$, we sample without the refusal instruction:
$\rvy_b^{(i)}\sim p_{\mathrm{ref}}(\cdot\mid\rvx_b^{(i)})$.

\paragraph{Training.}
Let $\pi_h \coloneqq p_{\text{ref}}(\cdot \mid I_\text{refusal}, \rvx_h)$ and $\pi_b \coloneqq p_{\text{ref}}(\cdot \mid \rvx_b)$ denote the harmful and benign generation distributions from the frozen reference model. Following~\citet{jiang-etal-2025-safechain}, we filter both sets with a safety guard $\varphi$ (\href{https://huggingface.co/meta-llama/Llama-Guard-3-8B}{Llama-Guard-3-8B}~\citep{metallamaguard3}), admitting only traces classified as safe. The student parameters $\theta$ are then optimized to minimize the negative log-likelihood of the accepted traces:
\begin{equation}
\label{eq:thinksafe}
\lambda \cdot
\mathbb{E}_{\substack{\rvx_h \sim \mu_h \\ \rvy_h \sim \pi_h}}
\left[\ell_\text{safe}(\rvx_h,\rvy_h)\right]
+
(1-\lambda)\cdot
\mathbb{E}_{\substack{\rvx_b \sim \mu_b \\ \rvy_b \sim \pi_b}}
\left[\ell_\text{safe}(\rvx_b,\rvy_b)\right],
\end{equation}
where $\ell_\text{safe}(\rvx,\rvy)
= -\log p_\theta(\rvy\mid\rvx)
\mathbf{1}\{\varphi(\rvx,\rvy)=1\}$ and
$\mu=(1-\lambda)\mu_b+\lambda\mu_h$ is the prompt distribution before filtering. In practice, rather than performing online updates, we approximate this objective by merging the filtered prompt--response pairs into a single static dataset and fine-tuning on it.

\section{Theoretical Analysis}
\label{sec:theory}

We analyze the population targets produced by safety filtering.
\textsc{ThinkSafe} uses the frozen student's native distribution on benign
prompts and refusal steering on harmful prompts. We show that its benign
targets minimize KL divergence to the frozen student among all
filter-accepted targets, and characterize when its average target KL is
smaller than that of teacher distillation. The comparison allows refusal
steering to change the relative probabilities of accepted responses.

\paragraph{Setup.}
Let $p_\theta$ denote the trainable student,
$p_{\mathrm{ref}}$ the student distribution frozen before safety
fine-tuning, and $p_T$ an external teacher distribution.
Let $\mu_b$ and $\mu_h$ be distributions on disjoint finite sets of benign
and harmful prompts, respectively, and let
$\mu=(1-\lambda)\mu_b+\lambda\mu_h$, where $\lambda\in[0,1]$.
Both methods are compared under these common prompt weights.
For a finite response space $\calY$, a fixed binary safety filter
$\varphi$, and any conditional generation distribution
$\pi(\cdot\mid\rvx)$, define
\begin{equation}
\label{eq:theory_filtered_target}
\begin{gathered}
\calS_{\rvx}
\coloneqq\{\rvy\in\calY:\varphi(\rvx,\rvy)=1\}, \qquad
\alpha_\pi(\rvx)
\coloneqq\sum_{\rvy\in\calS_{\rvx}}\pi(\rvy\mid\rvx),\\
\pi^+(\rvy\mid\rvx)
\coloneqq\frac{\pi(\rvy\mid\rvx)\varphi(\rvx,\rvy)}
{\alpha_\pi(\rvx)}.
\end{gathered}
\end{equation}
We assume positive acceptance for every distribution conditioned below
and $p_{\mathrm{ref}}(\rvy\mid\rvx)>0$ for all
$\rvy\in\calS_{\rvx}$. Write
$\alpha_{\mathrm{ref}}=\alpha_{p_{\mathrm{ref}}}$ and
$\pi_h(\cdot\mid\rvx)=
p_{\mathrm{ref}}(\cdot\mid I_{\mathrm{refusal}},\rvx)$.
The teacher target is $p_T^+$, whereas
\begin{equation}
\label{eq:theory_ts_target}
\pi_{\mathrm{TS}}^+(\cdot\mid\rvx)
=\begin{cases}
p_{\mathrm{ref}}^+(\cdot\mid\rvx), & \rvx\text{ benign},\\
\pi_h^+(\cdot\mid\rvx), & \rvx\text{ harmful}.
\end{cases}
\end{equation}
All logarithms are natural. Below, we suppress conditioning on $\rvx$
inside pointwise divergences and response expectations. These are
population statements about the specified generation distributions.

\paragraph{Minimal change on benign prompts.}
Conditioning the native student on filter acceptance gives the closest
accepted distribution in KL divergence.
\begin{lemma}[KL-optimal accepted target]
\label[lemma]{lem:projection}
For any prompt $\rvx$ and distribution $r$ supported on $\calS_{\rvx}$,
\begin{equation}
\label{eq:theory_projection}
\KL(r\|p_{\mathrm{ref}})
=-\log\alpha_{\mathrm{ref}}(\rvx)
+\KL(r\|p_{\mathrm{ref}}^+).
\end{equation}
Consequently, $p_{\mathrm{ref}}^+$ is the unique minimizer of
$\KL(r\|p_{\mathrm{ref}})$ over such distributions.
\end{lemma}
Thus, \textsc{ThinkSafe}'s benign targets introduce no excess KL beyond
the cost of conditioning on acceptance. They also satisfy
$\operatorname{TV}(p_{\mathrm{ref}}^+,p_{\mathrm{ref}})
=1-\alpha_{\mathrm{ref}}(\rvx)$, where
$\operatorname{TV}(P,Q)=\frac12\sum_{\rvy}|P(\rvy)-Q(\rvy)|$.
For any response score $s(\rvx,\rvy)\in[0,1]$, including a correctness
indicator, this implies
\begin{equation}
\label{eq:theory_benign_score_bound}
\left|\E_{\rvy\sim p_{\mathrm{ref}}^+}[s(\rvx,\rvy)]
-\E_{\rvy\sim p_{\mathrm{ref}}}[s(\rvx,\rvy)]\right|
\le 1-\alpha_{\mathrm{ref}}(\rvx).
\end{equation}
High benign acceptance therefore limits changes in the target's expected
score. This concerns the target distribution and the additional effect of
imperfect fine-tuning is quantified in \Cref{app:target_fitting}.

\paragraph{Comparison with teacher distillation.}
Define the average target divergences
\begin{equation}
\label{eq:theory_average_kls}
\begin{aligned}
\calK_T
\coloneqq\E_{\rvx\sim\mu}[\KL(p_T^+\|p_{\mathrm{ref}})], \quad
\calK_{\mathrm{TS}}
\coloneqq\E_{\rvx\sim\mu}
[\KL(\pi_{\mathrm{TS}}^+\|p_{\mathrm{ref}})].
\end{aligned}
\end{equation}
The teacher's excess divergence on each prompt type and the excess
divergence induced by refusal steering are
\begin{equation}
\label{eq:theory_excess_kls}
\begin{aligned}
G_j
\coloneqq\E_{\rvx\sim\mu_j}
[\KL(p_T^+\|p_{\mathrm{ref}}^+)],
\quad j\in\{b,h\},\qquad
H
\coloneqq\E_{\rvx\sim\mu_h}
[\KL(\pi_h^+\|p_{\mathrm{ref}}^+)].
\end{aligned}
\end{equation}
To quantify the effect of steering, define, for harmful $\rvx$ and
$\rvy\in\calS_{\rvx}$,
\begin{equation}
\label{eq:theory_steering_weights}
\begin{gathered}
w_{\rvx}(\rvy)\coloneqq\frac{\pi_h(\rvy\mid\rvx)}
{p_{\mathrm{ref}}(\rvy\mid\rvx)}, \qquad
m_{\rvx}
\coloneqq\E_{\rvy\sim p_{\mathrm{ref}}^+}[w_{\rvx}(\rvy)]
=\frac{\alpha_{\pi_h}(\rvx)}{\alpha_{\mathrm{ref}}(\rvx)},\\
V_{\rvx}
\coloneqq\frac{\E_{\rvy\sim p_{\mathrm{ref}}^+}
[(w_{\rvx}(\rvy)-m_{\rvx})^2]}{m_{\rvx}^2}.
\end{gathered}
\end{equation}
Here, $m_{\rvx}$ measures acceptance amplification, while $V_{\rvx}$
measures variation in the relative weights of accepted responses.

\begin{thm}[KL comparison with refusal steering]
\label[theorem]{thm:kl_advantage}
Under the setup above,
\begin{equation}
\label{eq:theory_exact_gap}
\calK_T-\calK_{\mathrm{TS}}
=(1-\lambda)G_b+\lambda G_h-\lambda H.
\end{equation}
Moreover,
\begin{equation}
\label{eq:theory_steering_bound}
H\le\E_{\rvx\sim\mu_h}[\log(1+V_{\rvx})].
\end{equation}
Hence $\calK_{\mathrm{TS}}<\calK_T$ whenever
\begin{equation}
\label{eq:theory_kl_condition}
(1-\lambda)G_b+\lambda G_h
>\lambda\E_{\rvx\sim\mu_h}[\log(1+V_{\rvx})].
\end{equation}
\end{thm}

The KL gap has two contributions. On benign prompts, self-generation avoids the teacher’s excess divergence \(G_b\). On harmful prompts, the gap depends on the difference between the teacher’s excess divergence \(G_h\) and the steering-induced divergence \(H\). It does not require exact invariance of the
accepted distribution: $\pi_h^+=p_{\mathrm{ref}}^+$ holds precisely when
$V_{\rvx}=0$. Increased acceptance alone does not ensure small
$V_{\rvx}$, and the strict KL advantage remains conditional on
\Cref{eq:theory_kl_condition}. Model size does not enter the decomposition,
so equal teacher and student sizes do not eliminate target mismatch.
Lower target KL describes distributional preservation. Proofs and the connection
to the fitted model are provided in \Cref{sec:app-proof}.

\section{Experiment}
\label{sec:experiment}

\subsection{Experimental Setup}
\paragraph{Implementation details.} 
Using the same set of prompts from the SafeChain dataset, we apply \textsc{ThinkSafe} to distilled models from the \href{https://huggingface.co/collections/deepseek-ai/deepseek-r1}{DeepSeek-R1-Distill series} (1.5B to 8B) \citep{guo_deepseek-r1_2025} and the \href{https://huggingface.co/collections/Qwen/qwen3}{Qwen3 family} (0.6B to 8B) \citep{yang2025qwen3technicalreport}. Details regarding our generated dataset are provided in~\Cref{sec:appendix_stats}. Based on the findings that LoRA \citep{hu2022lora} effectively preserves a model's intrinsic capabilities after fine-tuning \citep{Biderman2024LoRALL, Xue2025LoRAIA}, we adopt LoRA with rank $r=32$, scaling factor $\alpha=16$, and dropout rate~\citep{srivastava2014dropout} of $0.05$ applied to the query and value projections for all experimental configurations. For optimization, we use AdamW~\citep{adamw} with a base learning rate of $1 \times 10^{-5}$ and a cosine scheduler with a linear warmup over the first $10\%$ of training steps. While we strictly adhere to the original literature settings for the baselines, \textsc{ThinkSafe} is trained for 3 epochs, consistent with the SafeChain configuration. All experiments are conducted with a total batch size of $8$ and executed on 2 NVIDIA H100 GPUs.

\paragraph{Datasets.}
We use four challenging benchmarks to assess reasoning proficiency: GSM8K~\citep{gsm8k}, MATH500~\citep{math500}, AIME24~\citep{aime24}, and GPQA~\citep{gpqa}. We sample 8 responses per prompt and report the average pass@1 using \href{https://github.com/NovaSky-AI/SkyThought}{SkyThought}~\citep{skythought}. For safety, we evaluate on StrongReject~\citep{strongreject}, HarmBench~\citep{harmbench}, WildJailbreak~\citep{jiang_wildteaming_2024}, and XSTest~\citep{xstest}. For the first three, we sample a single response per prompt, evaluate harmfulness with \href{https://huggingface.co/meta-llama/Llama-Guard-3-8B}{Llama-Guard-3}, and report the ratio of harmful responses. For XSTest, we evaluate on the safe prompt subset to monitor over-refusal using \href{https://huggingface.co/allenai/wildguard}{WildGuard}~\citep{wildguard}. More details are in~\Cref{app:exp-detail}.


\begin{table*}[t]
    \caption{Results on \textbf{Qwen3} models. Safety reports harmful-response rates on HarmBench, StrongReject, and WildJailbreak, and over-refusal on benign XSTest prompts ($\downarrow$); their unweighted mean is the safety average. Reasoning reports average pass@1 from eight responses per prompt ($\uparrow$); its average is the unweighted mean of four benchmarks. All scores are percentages; averages are computed from the displayed component scores. Best and second-best distinct values within each model size are \textbf{bold} and \underline{underlined}; ties share a rank.}
    \vspace{-0.05in}
    \label{tab:qwen_result}
    \centering
    \small
    \resizebox{\textwidth}{!}{
    \begin{tabular}{clccc @{\hskip 0.2pt}c>{\columncolor{mygray}}ccccc >{\columncolor{mygray}}c}
        \toprule

            \multicolumn{1}{c}{} & \multicolumn{1}{c}{} & \multicolumn{5}{c}{\textbf{Safety ($\downarrow$)}} & \multicolumn{5}{c}{\multirow{2}{*}[-3pt]{\textbf{Reasoning (Avg pass@1, $\uparrow$)}}} \\

            \cmidrule[0.4pt](r){3-7}
            \multicolumn{1}{c}{} & \multicolumn{1}{c}{} & \multicolumn{3}{c}{\textbf{Harmfulness}} & \multicolumn{1}{c}{\textbf{\shortstack{Over-refusal}}} & \multicolumn{1}{c}{\multirow{2}{*}[-8.5pt]{\shortstack{Safety\\Average}}} & \multicolumn{5}{c}{} \\

            \cmidrule[0.4pt](r){3-5} \cmidrule[0.4pt](lr){6-6} \cmidrule[0.4pt](lr){8-12}
            \multicolumn{1}{c}{\raisebox{4.5pt}{\textbf{Size}}} & \multicolumn{1}{c}{\raisebox{4.5pt}{\textbf{Method}}} & \shortstack{Harm\\Bench} & \shortstack{Strong\\Reject} & \shortstack{Wild\\Jailbreak} & \multicolumn{1}{c}{\raisebox{4.5pt}{XSTest}} & \multicolumn{1}{c}{} & \shortstack{AIME\\2024} & \raisebox{4.5pt}{GSM8K} & \shortstack{MATH\\500} & \raisebox{4.5pt}{GPQA} & \multicolumn{1}{c}{{\shortstack{Reasoning\\Average}}} \\
            
            \midrule

            \multirow{7}{*}[-3pt]{\textbf{0.6B}} & Initial & 68.44 & 66.45 & 52.80 & \phantom{0}5.20 & 48.22 & \textbf{10.42} & \textbf{72.51} & \underline{71.73} & \underline{25.13} & \textbf{44.95} \\
            
            \noalign{\vspace{1pt}}
            \cdashline{2-12}
            \noalign{\vspace{1.5pt}}
            & DirectRefusal & \underline{43.85} & \textbf{11.82} & \textbf{36.30} & 83.60 & 43.89 & \phantom{0}5.83 & 64.30 & 67.53 & 24.81 & 40.62 \\
            
            & SafeChain & 58.64 & 72.84 & 49.60 & \textbf{\phantom{0}0.00} & 45.27 & \phantom{0}4.58 & 68.68 & 62.42 & 23.74 & 39.86 \\
            
            & STAR-1 & 56.64 & 38.02 & 50.60 & 22.40 & \underline{41.92} & \phantom{0}6.25 & 68.15 & 68.17 & 24.18 & 41.69 \\
            
            & SafePath & 67.61 & 60.06 & 52.80 & \phantom{0}\underline{4.40} & 46.22 & \phantom{0}7.92 & 71.26 & \textbf{71.77} & \textbf{26.07} & \underline{44.26} \\
            
            & SafeKey & 60.96 & 48.88 & 52.75 & 18.40 & 45.25 & \phantom{0}5.42 & 71.58 & 66.17 & 24.94 & 42.03 \\

            \rowpink
            \cellcolor{white}
            & \textbf{\textsc{ThinkSafe}} & \textbf{40.37} & \underline{33.87} & \underline{37.95} & \phantom{0}6.40 & \textbf{29.65} & \phantom{0}\underline{9.58} & \underline{72.36} & 70.65 & 23.30 & 43.97 \\

            \midrule

            \multirow{7}{*}[-3pt]{\textbf{1.7B}} & Initial & 52.66 & 36.10 & 51.10 & \textbf{\phantom{0}1.20} & 35.27 & \underline{44.58} & 84.31 & \underline{88.85} & \underline{41.73} & \underline{64.87} \\

            \noalign{\vspace{1pt}}
            \cdashline{2-12}
            \noalign{\vspace{1.5pt}}
            & DirectRefusal & 38.54 & \textbf{\phantom{0}5.75} & \underline{35.75} &  61.60 & 35.41 & 43.75 & 82.78 & 88.10 & 41.29 & 63.98 \\
            
            & SafeChain & 47.34 & 57.51 & 43.85 & \phantom{0}\underline{1.60} & 37.58 & 34.58 & \textbf{85.29} & 85.72 & 38.13 & 60.93 \\
            
            & STAR-1 & \underline{37.38} & \phantom{0}\underline{7.67} & 46.60 & 10.80 & \underline{25.61} & \textbf{46.25} & \underline{84.38} & 88.30 & 41.16 & \textbf{65.02} \\
            
            & SafePath & 54.15 & 36.42 & 49.30 & \textbf{\phantom{0}1.20} & 35.27 & 43.33 & 84.33 & 88.32 & \textbf{42.42} & 64.60 \\
            
            & SafeKey & 46.84 & 18.21 & 48.85 & \phantom{0}8.80 & 30.68 & 38.33 & 84.31 & 88.12 & 40.03 & 62.70 \\
            
            \rowpink
            \cellcolor{white}
            & \textbf{\textsc{ThinkSafe}} & \textbf{28.74} & \phantom{0}9.58 & \textbf{29.20} & \phantom{0}2.00 & \textbf{17.38} & 44.17 & 83.80 & \textbf{89.05} & 40.53 & 64.39 \\

            \midrule

            \multirow{7}{*}[-3pt]{\textbf{4B}} & Initial & 38.21 & \phantom{0}8.31 & 43.00 & \textbf{\phantom{0}0.80} & 22.58 & 67.50 & 84.69 & \underline{93.43} & 52.27 & 74.47 \\

            \noalign{\vspace{1pt}}
            \cdashline{2-12}
            \noalign{\vspace{1.5pt}}
            & DirectRefusal & 33.06 & \phantom{0}\underline{3.19} & 36.20 &  32.00 & 26.11 & 68.33 & 82.58 & 93.20 & 53.03 & 74.29 \\
            
            & SafeChain & 43.69 & 41.21 & 39.65 & \phantom{0}2.00 &  31.64 & 62.08 & 89.59 & 93.03 & 51.01 & 73.93 \\
            
            & STAR-1 & 33.72 & \phantom{0}5.75 & 35.15 & \phantom{0}6.80 & 20.36 & 62.50 & \underline{90.97} & 93.05 & 51.96 & 74.62 \\
            
            & SafePath & 37.71 & \phantom{0}7.35 & 42.45 & \phantom{0}\underline{1.60} & 22.28 & \underline{72.08} & 84.45 & 93.33 & \underline{53.54} & 75.85 \\
            
            & SafeKey & \underline{32.39} & \phantom{0}\underline{3.19} & \underline{32.95} & \textbf{\phantom{0}0.80} & \underline{17.33} & 67.08 & \textbf{91.79} & 92.87 & 51.83 & \underline{75.89} \\

            \rowpink
            \cellcolor{white}
            & \textbf{\textsc{ThinkSafe}} & \textbf{\phantom{0}9.63} & \textbf{\phantom{0}0.32} & \textbf{\phantom{0}7.45} & \phantom{0}2.80 & \textbf{\phantom{0}5.05} & \textbf{73.33} & 88.06 & \textbf{93.53} & \textbf{53.79} & \textbf{77.18} \\

            \midrule

            \multirow{7}{*}[-3pt]{\textbf{8B}} & Initial & 35.05 & \phantom{0}4.47 & 38.35 & \textbf{\phantom{0}0.40} & 19.57 & \underline{74.17} & 85.28 & \textbf{94.18} & 50.69 & 76.08 \\

            \noalign{\vspace{1pt}}
            \cdashline{2-12}
            \noalign{\vspace{1.5pt}}
            & DirectRefusal & \underline{24.42} & \phantom{0}1.92 & \underline{28.05} & 37.60 & 23.00 & \textbf{74.58} & 84.31 & 93.63 & 59.41 & 77.98 \\

            & SafeChain & 41.20 & 38.95 & 36.42 & \phantom{0}\underline{1.20} & 29.44 & 70.00 & \textbf{92.98} & 93.53 & 58.21 & \underline{78.68} \\
            & STAR-1 & \underline{24.42} & \phantom{0}\underline{1.28} & 29.25 & \phantom{0}6.80 & \underline{15.44} & 72.50 & 90.29 & 93.73 & 57.83 & 78.59 \\
            & SafePath & 35.22 & \phantom{0}6.71 & 39.45 & \phantom{0}\underline{1.20} & 20.65 & \textbf{74.58} & 84.89 & \underline{93.85} & \textbf{61.24} & 78.64 \\
            & SafeKey & 26.91 & \phantom{0}4.79 & 28.80 & \phantom{0}8.80 & 17.33 & 70.00 & \underline{92.44} & 93.30 & \underline{59.91} & \textbf{78.91} \\

            \rowpink
            \cellcolor{white}
            & \textbf{\textsc{ThinkSafe}} & \textbf{\phantom{0}9.14} & \textbf{\phantom{0}0.32} & \textbf{\phantom{0}7.35} & \phantom{0}\underline{1.20} & \textbf{\phantom{0}4.50} & 72.92 & 88.00 & 93.10 & 59.67 & 78.42 \\

        \bottomrule
    \end{tabular}
    }
\vspace{-0.1in}
\end{table*}
\vspace{-0.1in}

\paragraph{Baselines.}
We compare \textsc{ThinkSafe} against the following competitive fine-tuning baselines, following their original training settings. Detailed configurations are provided in \Cref{sec:appendix_baseline}.
\begin{itemize}[leftmargin=*, itemsep=0.1pt, topsep=0.9pt, parsep=5pt]
    \item \textbf{DirectRefusal} \citep{huang2025safetytaxsafetyalignment} adds a fixed thinking trace into existing refusal responses to harmful prompts, enforcing immediate refusals without reasoning.
    
    \item \textbf{SafeChain} \citep{jiang-etal-2025-safechain} distills both the intermediate reasoning chain and the final response from a larger teacher model, \href{https://huggingface.co/deepseek-ai/DeepSeek-R1}{DeepSeek-R1} (685B).
    
    \item \textbf{STAR-1}~\citep{wang2025star1saferalignmentreasoning} leverages a larger teacher to generate policy-guided reasoning traces, using an LLM-as-a-judge to select the top 1,000 examples for training.
    
    \item \textbf{SafePath} \citep{Jeung2025SAFEPATHPH} injects a safety cue (``Let's think about safety first'') at the beginning of reasoning, leaving the remainder of the generation unsupervised.
    
    \item \textbf{SafeKey} \citep{zhou-etal-2025-safekey} uses the same dataset as STAR-1 but employs an auxiliary dual-path safety head to strengthen safety signals in internal representations.
\end{itemize}

\begin{table*}[t]
    \caption{Results on \textbf{DeepSeek-R1-Distill} models. Safety reports harmful-response rates on HarmBench, StrongReject, and WildJailbreak, and over-refusal on benign XSTest prompts ($\downarrow$); their unweighted mean is the safety average. Reasoning reports average pass@1 from eight responses per prompt ($\uparrow$); its average is the unweighted mean of four benchmarks. All scores are percentages; averages are computed from the displayed component scores. Best and second-best distinct values within each model size are \textbf{bold} and \underline{underlined}; ties share a rank.}
    \vspace{-0.05in}
    \label{tab:deepseek_result}
    \centering
    \small
    \resizebox{\textwidth}{!}{
    \begin{tabular}{clccc @{\hskip 0.2pt}c>{\columncolor{mygray}}ccccc >{\columncolor{mygray}}c}
        \toprule

            \multicolumn{1}{c}{} & \multicolumn{1}{c}{} & \multicolumn{5}{c}{\textbf{Safety ($\downarrow$)}} & \multicolumn{5}{c}{\multirow{2}{*}[-3pt]{\textbf{Reasoning (Avg pass@1, $\uparrow$)}}} \\

            \cmidrule[0.4pt](r){3-7}
            \multicolumn{1}{c}{} & \multicolumn{1}{c}{} & \multicolumn{3}{c}{\textbf{Harmfulness}} & \multicolumn{1}{c}{\textbf{\shortstack{Over-refusal}}} & \multicolumn{1}{c}{\multirow{2}{*}[-8.5pt]{\shortstack{Safety\\Average}}} & \multicolumn{5}{c}{} \\

            \cmidrule[0.4pt](r){3-5} \cmidrule[0.4pt](lr){6-6} \cmidrule[0.4pt](lr){8-12}
            \multicolumn{1}{c}{\raisebox{4.5pt}{\textbf{Size}}} & \multicolumn{1}{c}{\raisebox{4.5pt}{\textbf{Method}}} & \shortstack{Harm\\Bench} & \shortstack{Strong\\Reject} & \shortstack{Wild\\Jailbreak} & \multicolumn{1}{c}{\raisebox{4.5pt}{XSTest}} & \multicolumn{1}{c}{} & \shortstack{AIME\\2024} & \raisebox{4.5pt}{GSM8K} & \shortstack{MATH\\500} & \raisebox{4.5pt}{GPQA} & \multicolumn{1}{c}{{\shortstack{Reasoning\\Average}}} \\
            
            \midrule

            \multirow{7}{*}[-3pt]{\textbf{1.5B}} & Initial & 67.28 & 82.11 & 51.55 & \textbf{\phantom{0}0.00} & 50.24 & 21.25 & \underline{82.42} & 79.45 & 31.94 & 53.77 \\
            
            \noalign{\vspace{1pt}}
            \cdashline{2-12}
            \noalign{\vspace{1.5pt}}
            & DirectRefusal & 66.11 & 82.75 & 50.45 & \phantom{0}8.40 & 51.93 & 19.17 & 81.06 & 78.55 & \textbf{32.70} & 52.87 \\

            & SafeChain & 59.30 & 76.68 & \underline{46.95} & \phantom{0}\underline{0.40} & 45.83 & \underline{24.17} & 80.47 & \underline{81.25} & 28.28 & 53.54 \\
            
            & STAR-1 & 62.79 & 77.00 & 49.65 & \phantom{0}1.20 & 47.66 & 17.08 & 81.38 & 79.07 & 31.25 & 52.20 \\
            
            & SafePath & 65.28 & 82.43 & 51.80 & \phantom{0}\underline{0.40} & 49.98 & \underline{24.17} & 82.37 & 79.57 & \underline{32.51} & \underline{54.66} \\
            
            & SafeKey & \underline{58.80} & \textbf{73.16} & 47.65 & \phantom{0}3.60 & \underline{45.80} & 19.58 & 81.25 & 78.02 & 28.72 & 51.89 \\

            \rowblue
            \cellcolor{white}
            & \textbf{\textsc{ThinkSafe}} & \textbf{52.99} & \underline{74.12} & \textbf{40.50} & \phantom{0}1.20 & \textbf{42.20} & \textbf{32.92} & \textbf{82.58} & \textbf{82.50} & 31.19 & \textbf{57.30} \\

            \midrule

            \multirow{7}{*}[-3pt]{\textbf{7B}} & Initial & 56.98 & 63.58 & 53.15 & \phantom{0}1.20 & 43.73 & 49.58 & \underline{90.32} & 90.18 & 46.65 & 69.18 \\

            \noalign{\vspace{1pt}}
            \cdashline{2-12}
            \noalign{\vspace{1.5pt}}
            & DirectRefusal & 52.33 & \textbf{33.55} & 50.20 & 43.60 & 44.92 & 47.50 & 88.27 & 89.82 & 44.95 & 67.64 \\

            & SafeChain & 51.00 & 54.63 & 45.85 & \phantom{0}\underline{0.40} & 37.97 & 49.17 & 89.75 & \underline{91.50} & \underline{46.78} & 69.30 \\
            
            & STAR-1 & 52.99 & 47.92 & 48.75 & \phantom{0}2.40 & 38.02 & 45.83 & \underline{90.32} & 90.58 & 46.02 & 68.19 \\
            
            & SafePath & 55.15 & 64.86 & 52.65 & \textbf{\phantom{0}0.00} & 43.17 & \textbf{52.08} & 89.71 & 90.62 & 46.15 & \textbf{69.64} \\
            
            & SafeKey & \underline{45.35} & \underline{33.87} & \underline{45.75} & \phantom{0}7.20 & \underline{33.04} & 43.75 & \textbf{90.58} & 89.90 & \textbf{47.29} & 67.88 \\
            
            \rowblue
            \cellcolor{white}
            & \textbf{\textsc{ThinkSafe}} & \textbf{40.20} & 41.85 & \textbf{35.40} & \phantom{0}\underline{0.40} & \textbf{29.46} & \underline{51.25} & 90.10 & \textbf{91.90} & 45.20 & \underline{69.61} \\

            \midrule

            \multirow{7}{*}[-3pt]{\textbf{8B}} & Initial & 52.33 & 53.99 & 49.70 & \textbf{\phantom{0}0.40} & 39.11 & \underline{47.50} & \textbf{87.74} & 87.38 & \textbf{48.11} & \textbf{67.68} \\
            
            \noalign{\vspace{1pt}}
            \cdashline{2-12}
            \noalign{\vspace{1.5pt}}
            & DirectRefusal & 32.39 & \textbf{\phantom{0}0.64} & 32.60 & 50.00 & 28.91 & 40.00 & 83.26 & 85.00 & 43.50 & 62.94 \\

            & SafeChain & 44.52 & 46.33 & 42.45 & \phantom{0}\underline{1.60} & 33.73 & 41.67 & 86.06 & 86.50 & 42.05 & 64.07 \\
            
            & STAR-1 & \textbf{21.26} & \phantom{0}\underline{3.51} & \textbf{17.60} & 12.00 & \textbf{13.59} & 40.42 & 87.28 & 86.65 & 43.69 & 64.51 \\
            
            & SafePath & 47.51 & 51.44 & 50.20 & \textbf{\phantom{0}0.40} & 37.39 & 43.33 & 87.45 & \underline{87.43} & \underline{47.41} & 66.41 \\
            
            & SafeKey & 32.72 & 11.82 & 30.85 & \phantom{0}8.00 & 20.85 & 35.83 & 87.41 & 85.80 & 42.49 & 62.88 \\
            
            \rowblue
            \cellcolor{white}
            & \textbf{\textsc{ThinkSafe}} & \underline{27.08} & 26.52 & \underline{21.15} & \phantom{0}\underline{1.60} & \underline{19.09} & \textbf{48.75} & \underline{87.55} & \textbf{87.70} & 46.28 & \underline{67.57} \\
            
        \bottomrule
    \end{tabular}
    }
\vspace{-0.1in}
\end{table*}
\subsection{Main Results}

\paragraph{Safety--reasoning trade-offs.} As summarized in \Cref{tab:qwen_result} and \Cref{tab:deepseek_result}, \textsc{ThinkSafe} achieves favorable safety--reasoning trade-offs across both the Qwen3 and DeepSeek-R1-Distill families. On Qwen3-4B, \textsc{ThinkSafe} drastically reduces the harmful-response rate on HarmBench from 38.21\% to \textbf{9.63\%} while simultaneously boosting average reasoning accuracy from 74.47\% to \textbf{77.18\%}. On DeepSeek-R1-Distill-1.5B, it improves reasoning from 53.77\% to \textbf{57.30\%} while reducing the safety average from 50.24\% to \textbf{42.20\%}. Even on larger models like Qwen3-8B and DeepSeek-R1-Distill-8B, \textsc{ThinkSafe} cuts the safety average by more than half (\eg, 19.57\% $\to$ \textbf{4.50\%} on Qwen3-8B) while retaining reasoning performance close to or above that of the initial model. It achieves the lowest safety average in six of the seven main configurations. On DeepSeek-R1-Distill-8B, STAR-1 has a lower safety average (13.59\% versus 19.09\%), while \textsc{ThinkSafe} retains higher reasoning accuracy (67.57\% versus 64.51\%). These results demonstrate the practical effectiveness of self-generation with refusal steering, whose target distortion is analyzed in \Cref{sec:theory}.

\vspace{-0.1in}
\paragraph{Performance of teacher-distillation methods.} Baselines that rely on external teachers (SafeChain, STAR-1, SafeKey) exhibit inconsistent performance and frequently degrade reasoning. For example, on Qwen3-0.6B, SafeChain drops average reasoning from 44.95\% to 39.86\%, and on Qwen3-1.7B from 64.87\% to 60.93\%. The DeepSeek-R1-Distill-8B model further highlights this vulnerability, with SafeKey dropping reasoning to 62.88\%, SafeChain to 64.07\%, and STAR-1 to 64.51\% from an initial 67.68\%. However, some teacher-distillation configurations improve reasoning, including Qwen3-8B. 

\paragraph{DirectRefusal and SafePath.} DirectRefusal suffers severe over-refusal and reasoning penalties in several settings: on Qwen3-0.6B, it degrades average reasoning to 40.62\% while exhibiting an 83.6\% refusal rate on benign XSTest prompts. SafePath better retains reasoning but provides limited safety improvements: on Qwen3-1.7B, its mean harmful-response rate across HarmBench, StrongReject, and WildJailbreak is 46.62\%, compared to \textsc{ThinkSafe}'s 22.51\%. This three-benchmark mean excludes XSTest. These comparisons motivate explicit safety reasoning supervision, although the methods also differ in training data and objectives.

\subsection{Comparison with Online Learning Algorithms}
\paragraph{GRPO baseline and $\textsc{ThinkSafe}+\KL$.} We compare \textsc{ThinkSafe} against an online RL baseline trained via GRPO~\citep{guo_deepseek-r1_2025}. The student $p_\theta$, initialized from Qwen3-0.6B, generates rollouts $\rvy\sim p_\theta(\cdot\mid \rvx)$ for prompts $\rvx\in\calD$ and is optimized using a combined reward: a safety reward $r_\text{safety}(\rvx,\rvy) \in[0,1]$ derived from the safety guard model and a format reward $r_\text{format}(\rvx,\rvy)\in\{0,1\}$. Further details on the GRPO objective, reward design, and hyperparameters appear in \Cref{sec:appendix_grpo}.

To study additional distributional regularization, we introduce $\textsc{ThinkSafe}+\KL$. At the population level, standard cross-entropy minimizes forward KL from the accepted target to the student. In this variant, we replace the loss on benign responses with a token-wise, full-vocabulary forward KL $\KL(p_\text{ref}\|p_\theta)$ evaluated along reference-generated trajectories. This regularizes the student toward the reference on benign queries. Its KL direction and training trajectories differ from those of GRPO, whose reverse KL regularizer is $\KL(p_\theta\|p_\text{ref})$.

\paragraph{On-policy self-distillation baseline.}
We further compare against an on-policy self-distillation (OPSD) baseline \citep{zhao2026selfdistilled}. The student $p_\theta$ is initialized from Qwen3-0.6B, and we use the same model as the teacher, with the refusal-oriented instruction $I_{\mathrm{refusal}}$ prepended only for harmful prompts. Specifically, the teacher is defined as
\begin{equation}
\label{eq:online-self-teacher}
    p_T(\cdot \mid \rvx, \rvy_{<t}) =
    \begin{cases}
        p_{\text{ref}}(\cdot \mid \rvx, \rvy_{<t}),
        & \rvx \in \calD_b, \\
        p_{\text{ref}}(\cdot \mid I_{\mathrm{refusal}}, \rvx, \rvy_{<t}),
        & \rvx \in \calD_h.
    \end{cases}
\end{equation}
For each prompt $\rvx \in \mathcal{D}$, the student samples an on-policy rollout $\rvy \sim p_\theta(\cdot \mid \rvx)$, and the teacher provides token-wise supervision by minimizing the full-vocabulary reverse KL divergence $\KL(p_\theta(\cdot \mid \rvx,\rvy_{<t}) \,\Vert\, p_T(\cdot \mid \rvx,\rvy_{<t}))$, evaluated at each token position $t$ along the sampled rollout. The teacher remains fixed throughout training, and no additional KL regularization is applied. Further details on the OPSD objective, implementation, and hyperparameters are provided in \Cref{sec:appendix_grpo}.


\begin{figure}[t]
    \centering
    \begin{minipage}{0.49\linewidth}
        \centering
        \includegraphics[width=1.0\linewidth]{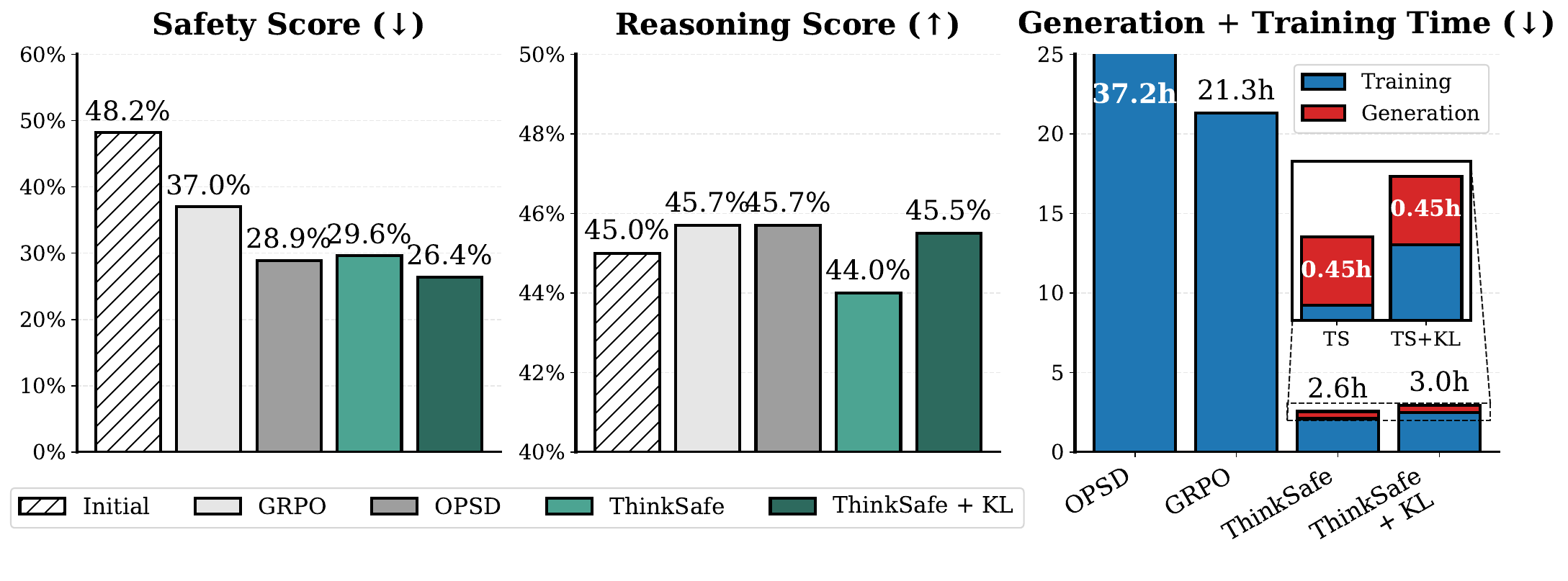}
        \vspace{-0.15in}
        \caption{Comparison against online methods.}
        \label{fig:RL}
    \end{minipage}
    \hfill
    \begin{minipage}{0.49\linewidth}
        \centering
        \includegraphics[width=1.0\linewidth]{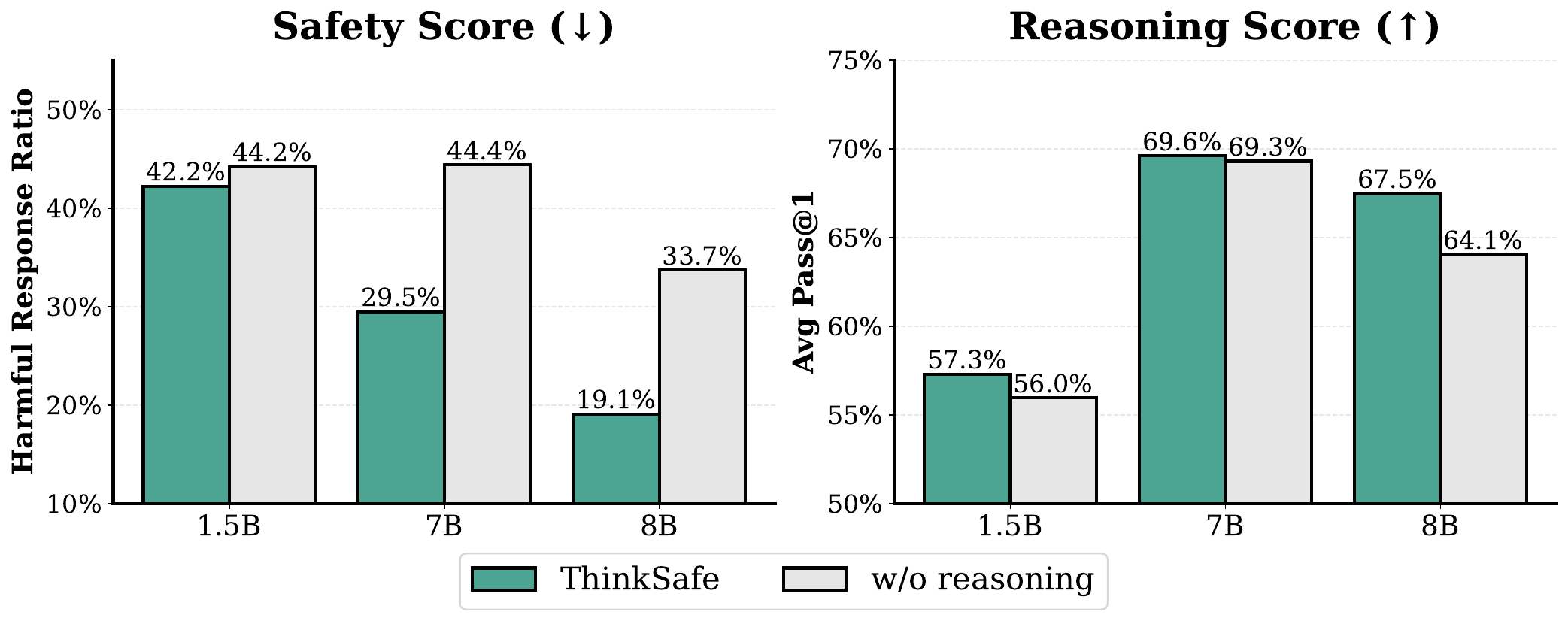}
        \vspace{-0.2in}
        \caption{Ablation of safety reasoning in R1 models.}
        \label{fig:think-ablation-r1}
    \end{minipage}
    \vspace{-0.1in}
\end{figure}
\paragraph{Results.} As shown in~\Cref{fig:RL}, \textsc{ThinkSafe} delivers a favorable balance of safety and efficiency compared to the evaluated online learning baselines. GRPO achieves a slight reasoning improvement but requires 21.3 hours of training, compared with 2.6 hours for \textsc{ThinkSafe}. Including 0.45 hours of data generation gives 3.05 hours for \textsc{ThinkSafe}, about one-seventh of GRPO's wall-clock time. OPSD requires 37.2 hours of training, despite achieving competitive safety and reasoning performance. \textsc{ThinkSafe} achieves a safety average of 29.65\%, compared to 37.0\% for GRPO and 28.9\% for OPSD, with reasoning averages of 43.97\%, 45.7\%, and 45.7\%, respectively. The $\textsc{ThinkSafe}+\KL$ variant further lowers the safety average to 26.4\% and raises reasoning to 45.5\%, close to GRPO's 45.7\%, using 3.45 hours including data generation. Compared with OPSD, it achieves better safety with comparable reasoning performance at less than one-tenth of the wall-clock cost. All wall-clock times reflect the reported experimental settings.

\subsection{Ablation Studies}

\paragraph{Effect of safety reasoning.} Prior work such as SafeChain \citep{jiang-etal-2025-safechain} observed that suppressing reasoning at inference can enhance safety. To test whether this holds at training time, we construct a ``w/o reasoning'' variant where reasoning traces are stripped from refusal responses $\rvy_h$ while benign responses $\rvy_b$ retain their full CoT, and train on the DeepSeek-R1-Distill family (Qwen3 results in \Cref{fig:think-ablation-qwen} of \Cref{sec:appendix_exp}). As shown in \Cref{fig:think-ablation-r1}, stripping safety reasoning sharply increases the safety average (\eg, 7B: 29.5\% $\to$ 44.4\%, 8B: 19.1\% $\to$ 33.7\%) and degrades reasoning itself: DeepSeek-R1-Distill-8B drops from 67.57\% to 64.1\% average pass@1. The effect is not uniform across models: on Qwen3-0.6B, removing safety reasoning lowers the safety average from 29.6\% to 27.0\% but reduces reasoning from 44.0\% to 42.0\%. We hypothesize that forcing the model to switch between ``thinking'' (benign) and ``not thinking'' (safety) may destabilize its learned chain-of-thought patterns, although this ablation does not isolate that mechanism.

\begin{figure}[t]
    \centering
    \begin{minipage}[t]{0.49\linewidth}
        \centering
        \includegraphics[width=1.0\linewidth]{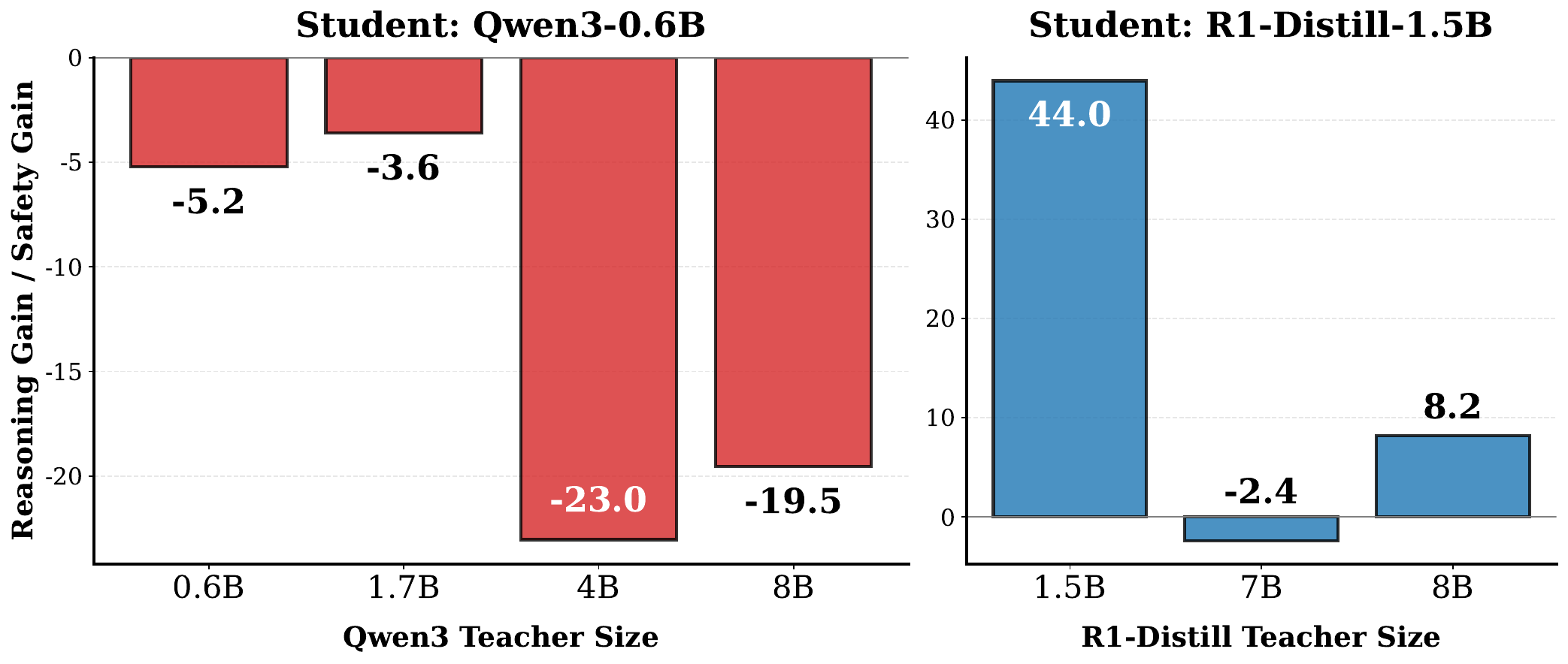}
        \vspace{-0.1in}
        \caption{Ratio of reasoning gain to safety gain for student models trained on data generated by teachers from the same model family.}
        \label{fig:teacher-family}
    \end{minipage}
    \hfill
    \begin{minipage}[t]{0.49\linewidth}
        \centering
        \includegraphics[width=1.0\linewidth]{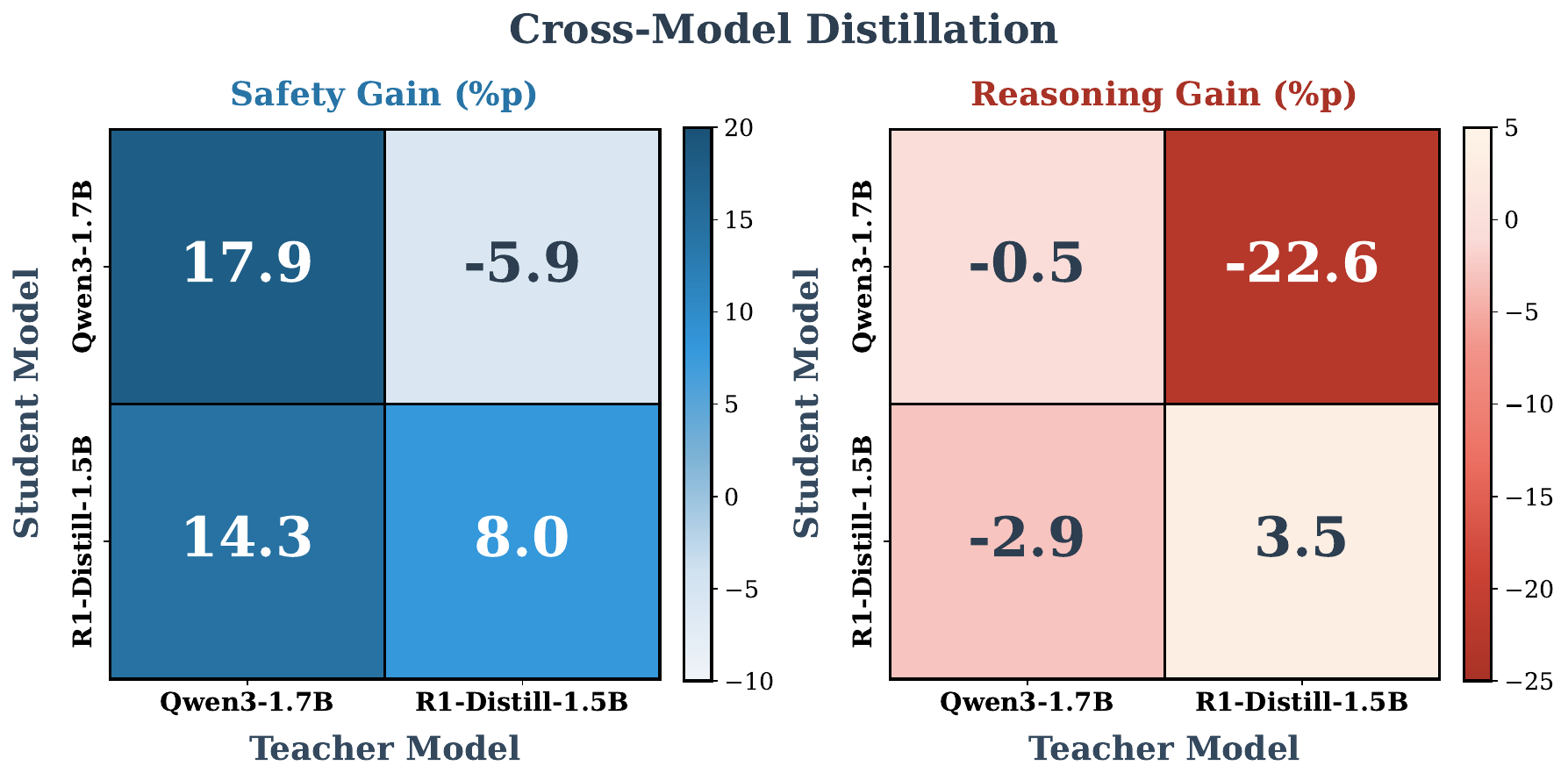}
        \vspace{-0.1in}
        \caption{\small Safety and reasoning performance gain using a different family of teacher model with similar size.}
        \label{fig:exchange}
    \end{minipage}
    \vspace{-0.1in}
\end{figure}

\paragraph{Refusal steering with external teachers.}
To examine the role of teacher size, we compare self-generation to steering-based distillation from teachers within the same release family. Using Qwen3-0.6B and DeepSeek-R1-Distill-1.5B as students, we generate safety data from larger teachers (Qwen3-1.7B/4B/8B and DeepSeek-R1-Distill-7B/8B). As shown in~\Cref{fig:teacher-family}, larger teachers do not uniformly improve the safety--reasoning trade-off. On Qwen3-0.6B, both self-generation and larger teachers yield negative ratios of reasoning gain to safety gain. On DeepSeek-R1-Distill-1.5B, the 8B teacher yields a positive ratio, but self-generation achieves a larger positive ratio.

To examine transfer between similarly sized models, we conduct a cross-model exchange using teachers from different families. We generate safety data from Qwen3-1.7B and DeepSeek-R1-Distill-1.5B via refusal steering and use each dataset to train the other student. As shown in~\Cref{fig:exchange}, cross-model training can improve safety (\eg, the safety average for DeepSeek-R1-Distill falls by 14.3 percentage points when trained on Qwen3-1.7B data) but consistently degrades reasoning (Qwen3-1.7B suffers a drop of 22.6 percentage points when trained on DeepSeek-R1-Distill data). These results show that similar model size does not eliminate reasoning loss after transfer.

\begin{figure}[t]
    \centering
    \begin{minipage}[t]{0.46\linewidth}
        \centering
        \includegraphics[width=1.0\linewidth]{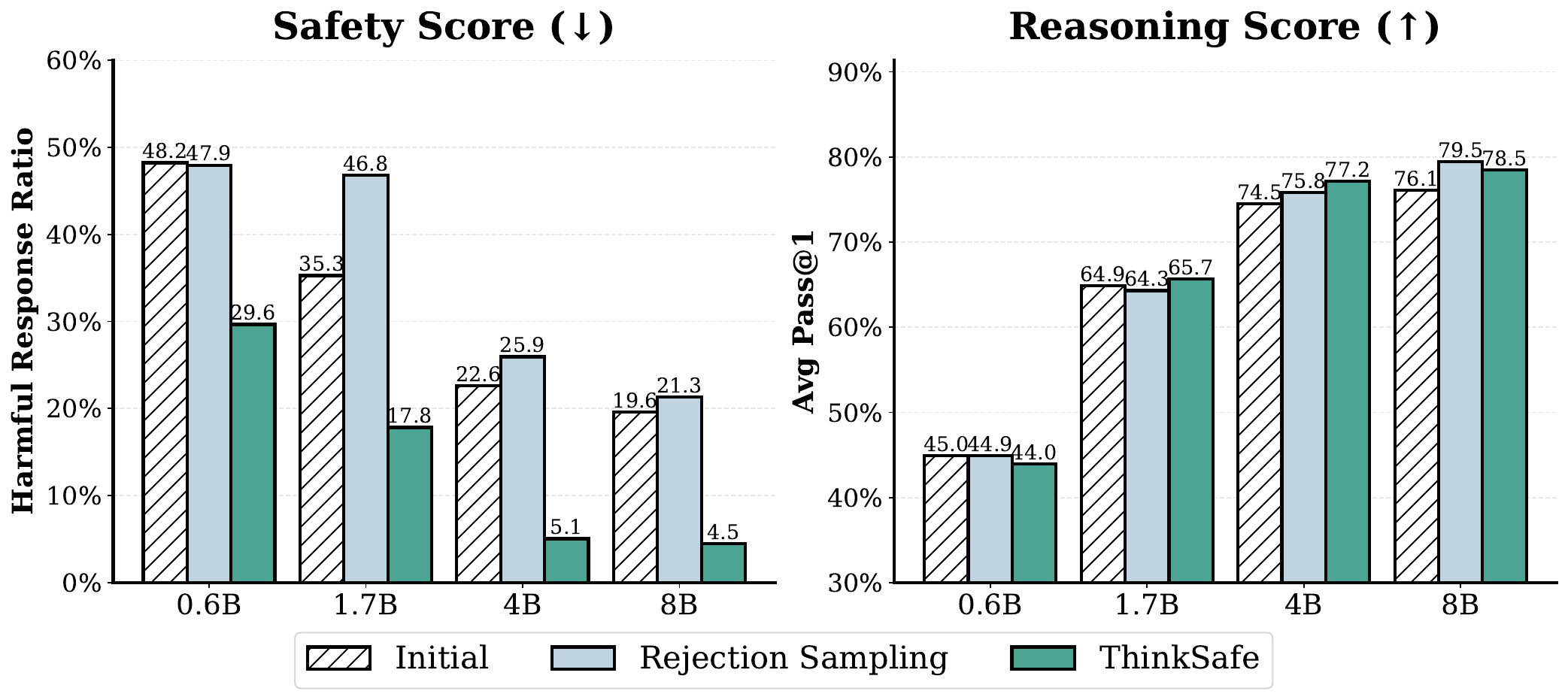}
        \vspace{-0.2in}
        \caption{We evaluate models trained on safety data generated via standard rejection sampling versus \textsc{ThinkSafe}.}
        \label{fig:rejection-sampling}
    \end{minipage}
    \hfill
    \begin{minipage}[t]{0.52\linewidth}
        \centering
        \includegraphics[width=1.0\linewidth]{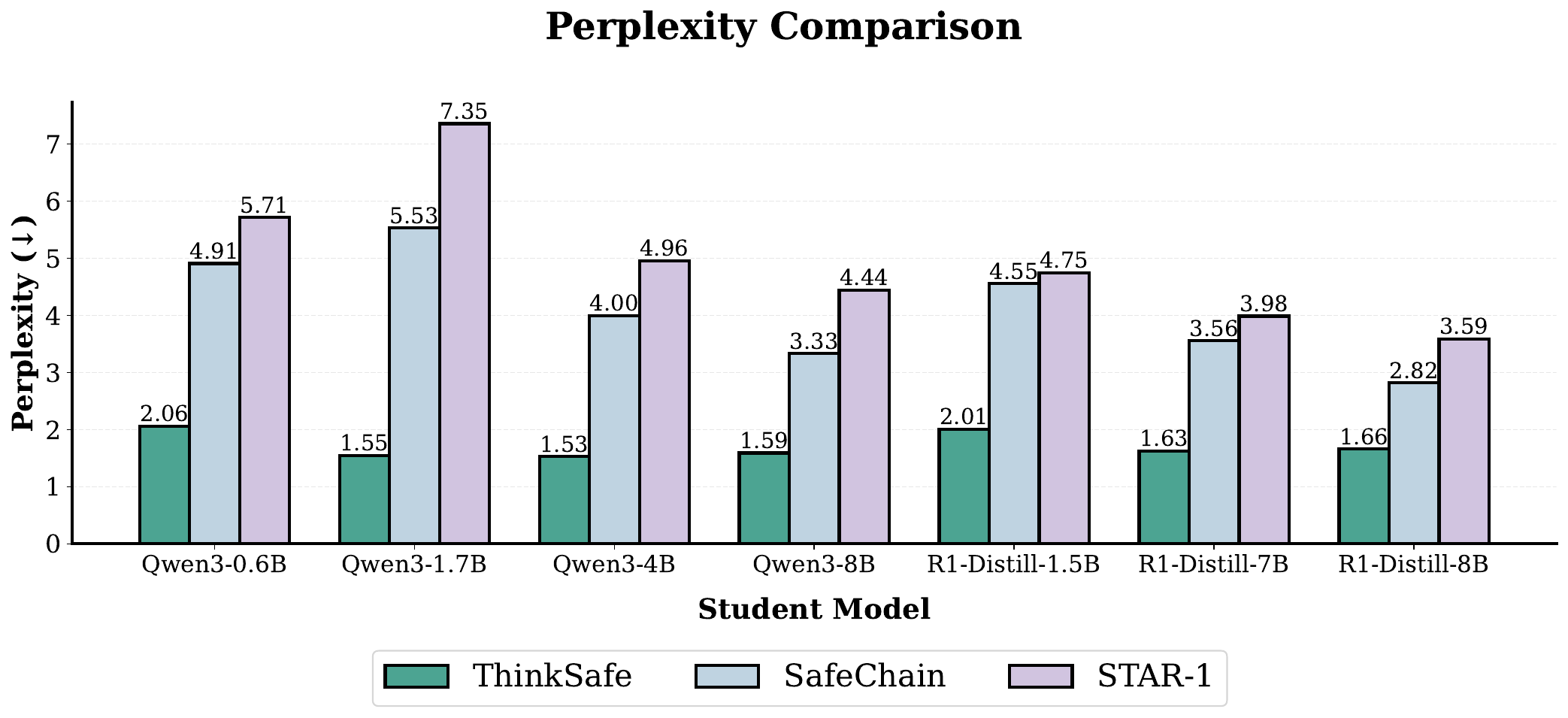}
        \vspace{-0.2in}
        \caption{Perplexity of generated safety dataset measured by the initial student models.}
        \label{fig:perplexity}
    \end{minipage}
    \vspace{-0.2in}
\end{figure}

\paragraph{Effect of refusal steering.} To examine the role of refusal steering, we compare \textsc{ThinkSafe} against an unsteered filtering baseline where responses are sampled directly from the frozen student without $I_\text{refusal}$. Following the SafeChain protocol~\citep{jiang-etal-2025-safechain}, we sample five responses per prompt and retain the prompt only if \emph{all five} are classified as safe, selecting one at random. As shown in \Cref{fig:rejection-sampling}, this unsteered baseline fails to produce comparable safety improvements. On Qwen3-8B, it yields a safety average of 21.3\%, compared with the initial 19.6\% and \textsc{ThinkSafe}'s 4.5\%. For independent samples from the unsteered generation distribution $\pi$, this protocol retains a prompt with probability $\alpha_\pi(\rvx)^5$. It therefore strongly favors prompts the model already handles safely. 

\subsection{Perplexity Analysis}
\Cref{sec:theory} characterizes the excess target KL, $\KL(\pi^+ \| p_\text{ref}^+)$, relative to the native accepted target $p_\text{ref}^+$. We examine the perplexity of each method's safety training data under the frozen student $p_\text{ref}$ as a likelihood-based diagnostic. At a fixed prompt, sequence-level cross-entropy satisfies $\mathcal{H}(\pi^+, p_\text{ref}) \coloneqq \E_{\rvy\sim\pi^+}[-\log p_\text{ref}(\rvy\mid\rvx)] = \mathcal{H}(\pi^+) + \KL(\pi^+ \| p_\text{ref})$, where $\mathcal{H}(\pi^+)$ is the Shannon entropy of the accepted source distribution. As shown in~\Cref{fig:perplexity}, \textsc{ThinkSafe} consistently achieves the lowest reported perplexity across the evaluated model sizes (\eg, 1.55 vs.\ 7.35 for STAR-1 on Qwen3-1.7B). These results indicate that \textsc{ThinkSafe}'s safety traces are more predictable under the frozen student than those of the evaluated teacher-distillation baselines. This supports self-generation as a source of safety supervision compatible with the student's native predictions. 
\section{Conclusion}

We presented \textsc{ThinkSafe}, a framework for self-generated safety alignment in reasoning models. Refusal steering elicits safety reasoning from a frozen copy of the student, and guard-based filtering produces a static dataset for offline fine-tuning. Our analysis identifies the native filter-accepted distribution as the unique KL-optimal target, bounds steering-induced distortion, and gives a sufficient condition for lower average target KL than teacher distillation. Experiments on the Qwen3 and DeepSeek-R1-Distill families demonstrate substantial safety improvements with competitive reasoning performance. On Qwen3-0.6B, \textsc{ThinkSafe} achieves better safety than GRPO and approaches OPSD's safety and reasoning performance at substantially lower wall-clock cost. Future work could extend the method to iterative self-training or use self-generated safety data to initialize RL-based policy optimization.

\paragraph{Limitations.} \textsc{ThinkSafe} assumes the student retains latent safety knowledge from prior alignment, which may not hold for base models without safety training. Its quality is also bounded by the external safety classifier used for filtering, and we approximate the on-policy objective with a static offline dataset that becomes increasingly off-policy during fine-tuning. Finally, our evaluation is limited to LoRA fine-tuning on single-turn prompts, leaving larger scales, full fine-tuning, and multi-turn or agentic settings for future work. See \Cref{sec:limitations} for extended discussion.

\section*{Acknowledgments}
This work was supported by Institute for Information \& communications Technology Planning \& Evaluation (IITP) grant funded by the Korea government (MSIT) (RS-2019-II190075, Artificial Intelligence Graduate School Program (KAIST), No.RS-2022-II220713, Meta-learning Applicable to Real-world Problems, No. RS-2020II200153, Penetration Security Testing of ML Model Vulnerabilities and Defense), National Research Foundation of Korea (NRF) grant funded by MSIT (No. RS-2023-00256259), and the “Advanced GPU Utilization Support Program” funded by the Government of the Republic of Korea (Ministry of Science and ICT).

\bibliography{reference}
\newpage
\appendix

\section{Proofs of the Theoretical Results}
\label{sec:app-proof}

We use the notation and assumptions of \Cref{sec:theory}. All distributions
in pointwise expressions are conditioned on the same prompt $\rvx$ unless
stated otherwise. We use $0\log(0/a)=0$ for $a>0$ and omit zero-weight
terms in expectations. Finiteness of the prompt and response spaces,
positive acceptance, and positivity of $p_{\mathrm{ref}}$ on
$\calS_{\rvx}$ ensure that the target divergences in the main comparison
are finite.

\subsection{Proof of \texorpdfstring{\Cref{lem:projection}}{the projection lemma}}
\label{app:proof_projection}

\begin{proof}
Fix $\rvx$. On $\calS_{\rvx}$,
$p_{\mathrm{ref}}=\alpha_{\mathrm{ref}}(\rvx)p_{\mathrm{ref}}^+$.
For any $r$ supported on $\calS_{\rvx}$, substitution gives
\begin{align*}
\KL(r\|p_{\mathrm{ref}})
&=\sum_{\rvy\in\calS_{\rvx}}r(\rvy\mid\rvx)
\log\frac{r(\rvy\mid\rvx)}
{\alpha_{\mathrm{ref}}(\rvx)p_{\mathrm{ref}}^+(\rvy\mid\rvx)}\\
&=-\log\alpha_{\mathrm{ref}}(\rvx)
+\KL(r\|p_{\mathrm{ref}}^+).
\end{align*}
The first term is independent of $r$, and the second is nonnegative,
with equality exactly when $r=p_{\mathrm{ref}}^+$.
\end{proof}

\paragraph{Benign response scores.}
For completeness, the total variation distance satisfies
\begin{align}
\operatorname{TV}(p_{\mathrm{ref}}^+,p_{\mathrm{ref}})
&=\frac12\sum_{\rvy\in\calS_{\rvx}}
\bigl(p_{\mathrm{ref}}^+(\rvy\mid\rvx)
-p_{\mathrm{ref}}(\rvy\mid\rvx)\bigr)\notag\\
&\quad+\frac12\sum_{\rvy\notin\calS_{\rvx}}
p_{\mathrm{ref}}(\rvy\mid\rvx)\notag\\
&=1-\alpha_{\mathrm{ref}}(\rvx).
\label{eq:app_benign_tv}
\end{align}
For any distributions $P,Q$ and score $s\in[0,1]$,
$|\E_P[s]-\E_Q[s]|\le\operatorname{TV}(P,Q)$: the positive and negative
parts of $P-Q$ each have total mass $\operatorname{TV}(P,Q)$.
Applying this inequality to \Cref{eq:app_benign_tv} proves
\Cref{eq:theory_benign_score_bound}, including the case
$\alpha_{\mathrm{ref}}(\rvx)=1$.

\subsection{Proof of \texorpdfstring{\Cref{thm:kl_advantage}}{the KL comparison theorem}}
\label{app:proof_kl_advantage}

\begin{proof}
\textbf{Conditional reweighting.}
For a harmful prompt, the mean of $w_{\rvx}$ is
\begin{align}
m_{\rvx}
&=\sum_{\rvy\in\calS_{\rvx}}
\frac{p_{\mathrm{ref}}(\rvy\mid\rvx)}{\alpha_{\mathrm{ref}}(\rvx)}
\frac{\pi_h(\rvy\mid\rvx)}{p_{\mathrm{ref}}(\rvy\mid\rvx)}\notag\\
&=\frac{\alpha_{\pi_h}(\rvx)}{\alpha_{\mathrm{ref}}(\rvx)}.
\label{eq:app_reweight_mean}
\end{align}
Consequently, on $\calS_{\rvx}$,
\begin{equation}
\label{eq:app_conditional_ratio}
\frac{\pi_h^+(\rvy\mid\rvx)}
{p_{\mathrm{ref}}^+(\rvy\mid\rvx)}
=\frac{w_{\rvx}(\rvy)}{m_{\rvx}}.
\end{equation}
The definition of $V_{\rvx}$ therefore gives
\begin{equation}
\label{eq:app_second_moment}
\E_{\rvy\sim p_{\mathrm{ref}}^+}
\left[\left(\frac{w_{\rvx}(\rvy)}{m_{\rvx}}\right)^2\right]
=1+V_{\rvx}.
\end{equation}
Applying Jensen's inequality under $\pi_h^+$ yields
\begin{align}
\KL(\pi_h^+\|p_{\mathrm{ref}}^+)
&=\E_{\rvy\sim\pi_h^+}
\left[\log\frac{w_{\rvx}(\rvy)}{m_{\rvx}}\right]\notag\\
&\le\log\E_{\rvy\sim\pi_h^+}
\left[\frac{w_{\rvx}(\rvy)}{m_{\rvx}}\right]\notag\\
&=\log\E_{\rvy\sim p_{\mathrm{ref}}^+}
\left[\left(\frac{w_{\rvx}(\rvy)}{m_{\rvx}}\right)^2\right]\notag\\
&=\log(1+V_{\rvx}).
\label{eq:app_pointwise_steering_bound}
\end{align}
Only positive-weight terms under $\pi_h^+$ are needed, so zeros of
$w_{\rvx}$ cause no difficulty. Averaging over $\mu_h$ proves
\Cref{eq:theory_steering_bound}.

\noindent\textbf{Exact comparison.}
Linearity of expectation under the mixture gives, for any function $f$,
\begin{equation}
\label{eq:app_mixture_expectation}
\E_{\rvx\sim\mu}[f(\rvx)]
=(1-\lambda)\E_{\rvx\sim\mu_b}[f(\rvx)]
+\lambda\E_{\rvx\sim\mu_h}[f(\rvx)].
\end{equation}
Applying \Cref{lem:projection} to both targets, then using
\Cref{eq:theory_ts_target,eq:theory_excess_kls}, yields
\begin{align}
\calK_T
&=\E_{\rvx\sim\mu}[-\log\alpha_{\mathrm{ref}}(\rvx)]
+(1-\lambda)G_b+\lambda G_h,
\label{eq:app_teacher_mixture}\\
\calK_{\mathrm{TS}}
&=\E_{\rvx\sim\mu}[-\log\alpha_{\mathrm{ref}}(\rvx)]
+\lambda H.
\label{eq:app_ts_mixture}
\end{align}
The benign excess KL for \textsc{ThinkSafe} is zero because its target is
$p_{\mathrm{ref}}^+$. Subtracting cancels the common conditioning cost
and proves \Cref{eq:theory_exact_gap}. The steering bound then gives
\begin{equation}
\label{eq:app_kl_lower_bound}
\calK_T-\calK_{\mathrm{TS}}
\ge (1-\lambda)G_b+\lambda G_h
-\lambda\E_{\rvx\sim\mu_h}[\log(1+V_{\rvx})].
\end{equation}
The right-hand side is positive under \Cref{eq:theory_kl_condition}.
\end{proof}

\paragraph{A scalar sufficient condition.}
If an additional moment bound
$\E_{\rvx\sim\mu_h}[V_{\rvx}]\le\eps^2$ is available, concavity of
the logarithm gives
\begin{equation}
\label{eq:app_average_steering_bound}
H\le\E_{\rvx\sim\mu_h}[\log(1+V_{\rvx})]
\le\log(1+\eps^2).
\end{equation}
Thus $(1-\lambda)G_b+\lambda G_h>\lambda\log(1+\eps^2)$ is also
sufficient. Since $G_h\ge0$, the more conservative condition
$(1-\lambda)G_b>\lambda\log(1+\eps^2)$ is sufficient as well.
Neither a small moment bound nor acceptance amplification is assumed
by the exact decomposition.

\paragraph{Acceptance and conditional invariance.}
\Cref{eq:app_reweight_mean} shows that steering increases acceptance
exactly when $m_{\rvx}>1$. By \Cref{eq:app_conditional_ratio}, it preserves
the accepted target exactly when $w_{\rvx}=m_{\rvx}$
$p_{\mathrm{ref}}^+$-almost surely, equivalently $V_{\rvx}=0$.
These are distinct properties: acceptance rates alone do not determine
the distortion within the accepted set.

\subsection{From population targets to the fitted model}
\label{app:target_fitting}

For each prompt $\rvx$, define the population fitting error
\begin{equation}
\label{eq:app_fitting_error}
\delta_\theta(\rvx)
\coloneqq
\KL\!\left(
\pi_{\mathrm{TS}}^+(\cdot\mid\rvx)
\,\|\,
p_\theta(\cdot\mid\rvx)
\right).
\end{equation}
This measures how closely the fitted model matches the accepted
target distribution. It equals the conditional population
negative log-likelihood minus the entropy of the target, rather
than the empirical training loss.
Throughout the pointwise derivations below, we fix $\rvx$ and
suppress conditioning on it when convenient.

\paragraph{Filter acceptance.}
The fitted model's acceptance probability is
\[
\alpha_{p_\theta}(\rvx)
=
\sum_{\rvy\in\calS_{\rvx}}
p_\theta(\rvy\mid\rvx).
\]
Suppose first that $\delta_\theta(\rvx)<\infty$.
Then $p_\theta(\rvy\mid\rvx)>0$ whenever
$\pi_{\mathrm{TS}}^+(\rvy\mid\rvx)>0$.
Otherwise, a positive-probability term in the KL would be infinite.
Since $\pi_{\mathrm{TS}}^+$ assigns total probability one to
$\calS_{\rvx}$, it follows that
$\alpha_{p_\theta}(\rvx)>0$.

We can therefore condition the fitted model on acceptance:
\[
p_\theta^+(\rvy\mid\rvx)
\coloneqq
\frac{
p_\theta(\rvy\mid\rvx)\varphi(\rvx,\rvy)
}{
\alpha_{p_\theta}(\rvx)
}.
\]
For $\rvy\in\calS_{\rvx}$, the filter equals one, so
\[
p_\theta(\rvy\mid\rvx)
=
\alpha_{p_\theta}(\rvx)
p_\theta^+(\rvy\mid\rvx).
\]
Substituting this equality into the definition of KL gives
\begin{align}
\delta_\theta(\rvx)
&=
\sum_{\rvy\in\calS_{\rvx}}
\pi_{\mathrm{TS}}^+(\rvy\mid\rvx)
\log
\frac{
\pi_{\mathrm{TS}}^+(\rvy\mid\rvx)
}{
p_\theta(\rvy\mid\rvx)
}
\notag\\
&=
\sum_{\rvy\in\calS_{\rvx}}
\pi_{\mathrm{TS}}^+(\rvy\mid\rvx)
\log
\frac{
\pi_{\mathrm{TS}}^+(\rvy\mid\rvx)
}{
\alpha_{p_\theta}(\rvx)
p_\theta^+(\rvy\mid\rvx)
}
\notag\\
&=
\sum_{\rvy\in\calS_{\rvx}}
\pi_{\mathrm{TS}}^+(\rvy\mid\rvx)
\log
\frac{
\pi_{\mathrm{TS}}^+(\rvy\mid\rvx)
}{
p_\theta^+(\rvy\mid\rvx)
}
\notag\\
&\qquad
-
\log\alpha_{p_\theta}(\rvx)
\sum_{\rvy\in\calS_{\rvx}}
\pi_{\mathrm{TS}}^+(\rvy\mid\rvx)
\notag\\
&=
\KL(\pi_{\mathrm{TS}}^+\|p_\theta^+)
-
\log\alpha_{p_\theta}(\rvx).
\label{eq:app_fitting_projection}
\end{align}
The final equality uses
$\sum_{\rvy\in\calS_{\rvx}}
\pi_{\mathrm{TS}}^+(\rvy\mid\rvx)=1$.
Terms with zero target probability are omitted.

Since KL divergence is nonnegative, we obtain the following
sequence of implications:
\begin{align*}
\delta_\theta(\rvx)
&\ge -\log\alpha_{p_\theta}(\rvx),\\
-\delta_\theta(\rvx)
&\le \log\alpha_{p_\theta}(\rvx),\\
e^{-\delta_\theta(\rvx)}
&\le \alpha_{p_\theta}(\rvx),\\
1-\alpha_{p_\theta}(\rvx)
&\le 1-e^{-\delta_\theta(\rvx)}.
\end{align*}
The second line follows by multiplying by $-1$ and reversing
the inequality. The third follows because the exponential
function is increasing. The fourth follows by subtracting
both sides from one.

Because the filter is binary,
\[
\Pr_{\rvy\sim p_\theta(\cdot\mid\rvx)}
[\varphi(\rvx,\rvy)=0]
=
1-\alpha_{p_\theta}(\rvx).
\]
Consequently,
\begin{equation}
\label{eq:app_fitted_safety}
\Pr_{\rvy\sim p_\theta(\cdot\mid\rvx)}
[\varphi(\rvx,\rvy)=0]
\le
1-e^{-\delta_\theta(\rvx)}.
\end{equation}
If $\delta_\theta(\rvx)=\infty$, the right-hand side equals one
under the convention $e^{-\infty}=0$, so the inequality remains
valid.

In particular, if
$\delta_\theta(\rvx)<-\log\alpha_{\mathrm{ref}}(\rvx)$, then
\[
\alpha_{p_\theta}(\rvx)
\ge
e^{-\delta_\theta(\rvx)}
>
e^{\log\alpha_{\mathrm{ref}}(\rvx)}
=
\alpha_{\mathrm{ref}}(\rvx).
\]
Thus, this condition guarantees higher filter acceptance than
the reference. The argument applies to any target supported
on $\calS_{\rvx}$, including filtered teacher targets.

\paragraph{Benign response preservation.}
For a benign prompt,
$\pi_{\mathrm{TS}}^+=p_{\mathrm{ref}}^+$, and hence
\[
\delta_\theta(\rvx)
=
\KL(p_{\mathrm{ref}}^+\|p_\theta).
\]
We first bound the distance between the fitted model and this
target, then account for the change introduced by filtering.

Pinsker's inequality, proved below, states that
\[
\operatorname{TV}(P,Q)
\le
\sqrt{\frac{\KL(P\|Q)}{2}},
\]
where logarithms are natural and
$\operatorname{TV}(P,Q)=\frac12\sum_{\rvy}|P(\rvy)-Q(\rvy)|$.
Applying it with $P=p_{\mathrm{ref}}^+$ and $Q=p_\theta$ gives
\begin{align*}
\operatorname{TV}(p_\theta,p_{\mathrm{ref}}^+)
&=
\operatorname{TV}(p_{\mathrm{ref}}^+,p_\theta)\\
&\le
\sqrt{
\frac{\KL(p_{\mathrm{ref}}^+\|p_\theta)}{2}
}\\
&=
\sqrt{\frac{\delta_\theta(\rvx)}{2}}.
\end{align*}
Only the arguments of total variation are exchanged here.
The direction of the KL remains
$p_{\mathrm{ref}}^+\|p_\theta$.

Next, expand the total variation distance to the reference:
\begin{align*}
\operatorname{TV}(p_\theta,p_{\mathrm{ref}})
&=
\frac12\sum_{\rvy\in\calY}
\left|
p_\theta(\rvy\mid\rvx)
-
p_{\mathrm{ref}}(\rvy\mid\rvx)
\right|\\
&=
\frac12\sum_{\rvy\in\calY}
\Bigl|
\bigl(
p_\theta(\rvy\mid\rvx)
-
p_{\mathrm{ref}}^+(\rvy\mid\rvx)
\bigr)
+
\bigl(
p_{\mathrm{ref}}^+(\rvy\mid\rvx)
-
p_{\mathrm{ref}}(\rvy\mid\rvx)
\bigr)
\Bigr|\\
&\le
\frac12\sum_{\rvy\in\calY}
\left|
p_\theta(\rvy\mid\rvx)
-
p_{\mathrm{ref}}^+(\rvy\mid\rvx)
\right| +
\frac12\sum_{\rvy\in\calY}
\left|
p_{\mathrm{ref}}^+(\rvy\mid\rvx)
-
p_{\mathrm{ref}}(\rvy\mid\rvx)
\right|\\
&=
\operatorname{TV}(p_\theta,p_{\mathrm{ref}}^+)
+
\operatorname{TV}(p_{\mathrm{ref}}^+,p_{\mathrm{ref}}).
\end{align*}
The inequality applies $|a+b|\le|a|+|b|$ separately to each
response.

Combining this calculation with Pinsker's inequality and
$\operatorname{TV}(p_{\mathrm{ref}}^+,p_{\mathrm{ref}})
=1-\alpha_{\mathrm{ref}}(\rvx)$ from
\Cref{eq:app_benign_tv} yields
\begin{equation}
\label{eq:app_fitted_benign_tv}
\operatorname{TV}(p_\theta,p_{\mathrm{ref}})
\le
\sqrt{\frac{\delta_\theta(\rvx)}{2}}
+
1-\alpha_{\mathrm{ref}}(\rvx).
\end{equation}
The first term bounds the error in fitting the accepted target.
The second is the exact change caused by filtering the native
reference distribution.

\paragraph{From total variation to response scores.}
We next justify why total variation bounds changes in any
score taking values in $[0,1]$.

For arbitrary distributions $P,Q$ on $\calY$, define
\[
\calA
\coloneqq
\{\rvy\in\calY:P(\rvy)\ge Q(\rvy)\}.
\]
Since both distributions sum to one,
\[
\sum_{\rvy\in\calY}(P(\rvy)-Q(\rvy))=0.
\]
Therefore, the total positive difference equals the magnitude of the total negative difference:
\begin{align*}
\sum_{\rvy\in\calA}(P(\rvy)-Q(\rvy))
&=
-\sum_{\rvy\notin\calA}(P(\rvy)-Q(\rvy))\\
&=
\operatorname{TV}(P,Q).
\end{align*}

Fix $\rvx$ and let $s(\rvx,\rvy)\in[0,1]$.
On $\calA$, the difference $P-Q$ is nonnegative, so multiplying
it by $s$ cannot increase it.
On the complement, the difference is negative, so its
contribution to an upper bound can be discarded.
Hence,
\begin{align*}
\E_{\rvy\sim P}[s(\rvx,\rvy)]
-
\E_{\rvy\sim Q}[s(\rvx,\rvy)]
&=
\sum_{\rvy\in\calY}
s(\rvx,\rvy)(P(\rvy)-Q(\rvy))\\
&\le
\sum_{\rvy\in\calA}(P(\rvy)-Q(\rvy))\\
&=
\operatorname{TV}(P,Q).
\end{align*}
Conversely, the contribution on $\calA$ is nonnegative, and
multiplying a negative difference by $s\le1$ makes it no
smaller. Thus,
\begin{align*}
\E_{\rvy\sim P}[s(\rvx,\rvy)]
-
\E_{\rvy\sim Q}[s(\rvx,\rvy)]
&\ge
\sum_{\rvy\notin\calA}(P(\rvy)-Q(\rvy))\\
&=
-\operatorname{TV}(P,Q).
\end{align*}
Combining the upper and lower bounds gives
\[
\left|
\E_{\rvy\sim P}[s(\rvx,\rvy)]
-
\E_{\rvy\sim Q}[s(\rvx,\rvy)]
\right|
\le
\operatorname{TV}(P,Q).
\]

Applying this result to
$P=p_\theta(\cdot\mid\rvx)$ and
$Q=p_{\mathrm{ref}}(\cdot\mid\rvx)$, followed by
\Cref{eq:app_fitted_benign_tv}, gives
\begin{equation}
\label{eq:app_fitted_benign_score}
\begin{aligned}
&\left|
\E_{\rvy\sim p_\theta}[s(\rvx,\rvy)]
-
\E_{\rvy\sim p_{\mathrm{ref}}}[s(\rvx,\rvy)]
\right|\\
&\qquad\le
1-\alpha_{\mathrm{ref}}(\rvx)
+
\sqrt{\frac{\delta_\theta(\rvx)}{2}}.
\end{aligned}
\end{equation}
A correctness indicator is one possible choice of $s$.

\paragraph{Averaging over benign prompts.}
For this calculation, define the per-prompt score difference
\[
d_\theta(\rvx)
\coloneqq
\E_{\rvy\sim p_\theta(\cdot\mid\rvx)}
[s(\rvx,\rvy)]
-
\E_{\rvy\sim p_{\mathrm{ref}}(\cdot\mid\rvx)}
[s(\rvx,\rvy)].
\]
Writing out the joint expectations shows that their difference
equals the average of these per-prompt differences:
\begin{align*}
&
\E_{\substack{
\rvx\sim\mu_b\\
\rvy\sim p_\theta(\cdot\mid\rvx)
}}
[s(\rvx,\rvy)]
-
\E_{\substack{
\rvx\sim\mu_b\\
\rvy\sim p_{\mathrm{ref}}(\cdot\mid\rvx)
}}
[s(\rvx,\rvy)]\\
&\quad=
\sum_{\rvx}\mu_b(\rvx)
\sum_{\rvy}
s(\rvx,\rvy)
\bigl(
p_\theta(\rvy\mid\rvx)
-
p_{\mathrm{ref}}(\rvy\mid\rvx)
\bigr)\\
&\quad=
\sum_{\rvx}\mu_b(\rvx)d_\theta(\rvx).
\end{align*}
Here, sums over $\rvx$ run over the finite support of $\mu_b$.

Applying the triangle inequality to the finite sum and then
using \Cref{eq:app_fitted_benign_score} gives
\begin{align*}
\left|
\sum_{\rvx}\mu_b(\rvx)d_\theta(\rvx)
\right|
&\le
\sum_{\rvx}\mu_b(\rvx)|d_\theta(\rvx)|\\
&\le
\sum_{\rvx}\mu_b(\rvx)
\left(
1-\alpha_{\mathrm{ref}}(\rvx)
+
\sqrt{\frac{\delta_\theta(\rvx)}{2}}
\right)\\
&=
\E_{\rvx\sim\mu_b}[1-\alpha_{\mathrm{ref}}(\rvx)]
+
\E_{\rvx\sim\mu_b}
\left[
\sqrt{\frac{\delta_\theta(\rvx)}{2}}
\right].
\end{align*}

It remains to bound the average square root.
For finite average fitting error, the Cauchy--Schwarz inequality
gives
\begin{align*}
\left(
\sum_{\rvx}\mu_b(\rvx)
\sqrt{\frac{\delta_\theta(\rvx)}{2}}
\right)^2
&=
\left(
\sum_{\rvx}
\sqrt{\mu_b(\rvx)}
\sqrt{\frac{\mu_b(\rvx)\delta_\theta(\rvx)}{2}}
\right)^2\\
&\le
\left(\sum_{\rvx}\mu_b(\rvx)\right)
\left(
\sum_{\rvx}
\frac{\mu_b(\rvx)\delta_\theta(\rvx)}{2}
\right)\\
&=
\frac12
\E_{\rvx\sim\mu_b}[\delta_\theta(\rvx)].
\end{align*}
The last equality uses $\sum_{\rvx}\mu_b(\rvx)=1$.
Taking square roots gives
\[
\E_{\rvx\sim\mu_b}
\left[
\sqrt{\frac{\delta_\theta(\rvx)}{2}}
\right]
\le
\sqrt{
\frac12
\E_{\rvx\sim\mu_b}[\delta_\theta(\rvx)]
}.
\]
This is also Jensen's inequality for the concave square-root
function. If the average fitting error is infinite, the
inequality holds trivially with an infinite right-hand side.

Substituting this bound into the preceding calculation yields
\begin{equation}
\label{eq:app_average_benign_score}
\begin{aligned}
&\left|
\E_{\substack{
\rvx\sim\mu_b\\
\rvy\sim p_\theta(\cdot\mid\rvx)
}}
[s(\rvx,\rvy)]
-
\E_{\substack{
\rvx\sim\mu_b\\
\rvy\sim p_{\mathrm{ref}}(\cdot\mid\rvx)
}}
[s(\rvx,\rvy)]
\right|\\
&\qquad\le
\E_{\rvx\sim\mu_b}[1-\alpha_{\mathrm{ref}}(\rvx)]
+
\sqrt{
\frac12
\E_{\rvx\sim\mu_b}[\delta_\theta(\rvx)]
}.
\end{aligned}
\end{equation}

\paragraph{Derivation of Pinsker's inequality.}
For completeness, we derive the KL-to-TV inequality used above.
Let $P,Q$ be distributions on a finite space and define
\[
\calA\coloneqq\{\rvy:P(\rvy)\ge Q(\rvy)\},
\qquad
a\coloneqq P(\calA),
\qquad
b\coloneqq Q(\calA).
\]
The positive-difference identity established above gives
\[
a-b
=
\sum_{\rvy\in\calA}(P(\rvy)-Q(\rvy))
=
\operatorname{TV}(P,Q),
\]
so $a\ge b$.

First suppose $0<a,b<1$.
Conditioning on $\calA$ and its complement decomposes the KL as
\begin{align*}
\KL(P\|Q)
&=
a\log\frac{a}{b}
+
(1-a)\log\frac{1-a}{1-b}\\
&\quad+
a\,
\KL\!\left(
P(\cdot\mid\calA)
\,\|\,
Q(\cdot\mid\calA)
\right)\\
&\quad+
(1-a)\,
\KL\!\left(
P(\cdot\mid\calA^c)
\,\|\,
Q(\cdot\mid\calA^c)
\right)\\
&\ge
a\log\frac{a}{b}
+
(1-a)\log\frac{1-a}{1-b}.
\end{align*}
The decomposition follows by writing
$P(\rvy)=aP(\rvy\mid\calA)$ and
$Q(\rvy)=bQ(\rvy\mid\calA)$ on $\calA$, and using the analogous
factorization on $\calA^c$.
The inequality discards two nonnegative conditional KL terms.

For fixed $b\in(0,1)$, define
\[
F_b(u)
\coloneqq
u\log\frac{u}{b}
+
(1-u)\log\frac{1-u}{1-b}.
\]
Direct differentiation gives
\[
F_b(b)=0,
\qquad
F_b'(b)=0,
\qquad
F_b''(u)
=
\frac1u+\frac1{1-u}
=
\frac1{u(1-u)}
\ge4,
\]
where the last inequality uses $u(1-u)\le1/4$.
For $u\ge b$, integrating this bound gives
\[
F_b'(u)
=
F_b'(b)+\int_b^u F_b''(v)\,dv
\ge
4(u-b).
\]
Integrating once more, using $a\ge b$, yields
\[
F_b(a)
=
F_b(b)+\int_b^a F_b'(u)\,du
\ge
\int_b^a4(u-b)\,du
=
2(a-b)^2.
\]
Consequently,
\[
\KL(P\|Q)
\ge
F_b(a)
\ge
2(a-b)^2
=
2\operatorname{TV}(P,Q)^2.
\]
Rearranging establishes
\[
\operatorname{TV}(P,Q)
\le
\sqrt{\frac{\KL(P\|Q)}{2}}.
\]
The case $a=1$ with $b\in(0,1)$ follows by continuity.
If $b=0$ and $a>0$, then $\KL(P\|Q)=\infty$, and the inequality
is trivial. If $a=b=0$ or $b=1$, then $a-b=0$ and the TV is zero.
These cases cover the remaining boundaries.

\paragraph{Interpretation.}
High benign acceptance and small average population fitting
error suffice to preserve average benign response scores.
The inequalities remain valid when fitting error is large,
but they may then be uninformative.
This analysis does not establish that finite-data fine-tuning
achieves small population fitting error or that preservation
transfers to a different prompt distribution.
The acceptance guarantee is measured by the fixed filter
$\varphi$.

\section{\textsc{ThinkSafe}}
\subsection{Sampling details}
All prompts used in \textsc{ThinkSafe} are from the \href{https://huggingface.co/datasets/UWNSL/SafeChain}{SafeChain} dataset and processed by each model using our \textsc{ThinkSafe} framework. For Qwen3 model family, we sample one response for each prompt with top-p 0.95, top-k 20, temperature 0.6, and maximum token limit 16,384. For DeepSeek-R1-Distill family, we use the same hyperparameter except that top-k is set to 0. Responses that are classified as unsafe by the Llama-Guard-3, denoted as $\varphi$, are excluded from the analysis. \Cref{tab:refusal_rate} shows the ratio of filtered samples per model. Larger Qwen3 models (4B and 8B) retain over 99\% of both benign and harmful samples, while the smaller Qwen3-0.6B and Qwen3-1.7B filter out 4.84\% and 3.29\% of harmful samples, respectively. The R1-Distill-1.5B model exhibits substantially higher filtering rates. Overall, larger models tend to preserve a greater portion of the original data, suggesting more stable and consistent response distributions after filtering.

\begin{table}[h]
    \caption{Filtered ratio ($\%$) per model.}
    \label{tab:refusal_rate}
    \centering
    \small
    \resizebox{0.6\textwidth}{!}{
    \begin{tabular}{lccccccc}
        \toprule
            \multirow{2}{*}[-5pt]{\shortstack{Sample category}} & \multicolumn{4}{c}{Qwen3} & \multicolumn{3}{c}{DeepSeek-R1-Distill} \\

            \cmidrule[0.4pt](lr){2-5} \cmidrule[0.4pt](lr){6-8}
            & 0.6B & 1.7B & 4B & 8B & 1.5B & 7B & 8B \\

            \midrule

            \multicolumn{1}{c}{Benign} & 1.07 & 0.93 & 0.63 & 0.74 & 11.01 & 2.69 & 0.81 \\
            
            \multicolumn{1}{c}{Harmful} & 4.84 & 3.29 & 0.35 & 0.16 & 12.86 & 2.93 & 0.48 \\
            
        \bottomrule
    \end{tabular}}
\vspace{-0.1in}
\end{table}

\subsection{Statistics}
\label{sec:appendix_stats}
We report the statistics of the responses generated by \textsc{ThinkSafe} in~\Cref{fig:statistics_qwen,fig:statistics_R1}. Here, $N_h$ and $N_b$ denote the numbers of harmful and benign prompts, respectively, while $\mu_h$ and $\mu_b$ represent the average response lengths (in tokens) for harmful and benign queries.

Across both the Qwen3 and R1-distilled model series, benign responses consistently exhibit longer generation lengths than harmful ones, reflecting the presence of more detailed reasoning traces. Moreover, as model size increases, both harmful and benign responses tend to become longer and more stable in distribution.

\begin{figure}[ht]
    \centering
    \includegraphics[width=1.0\linewidth]{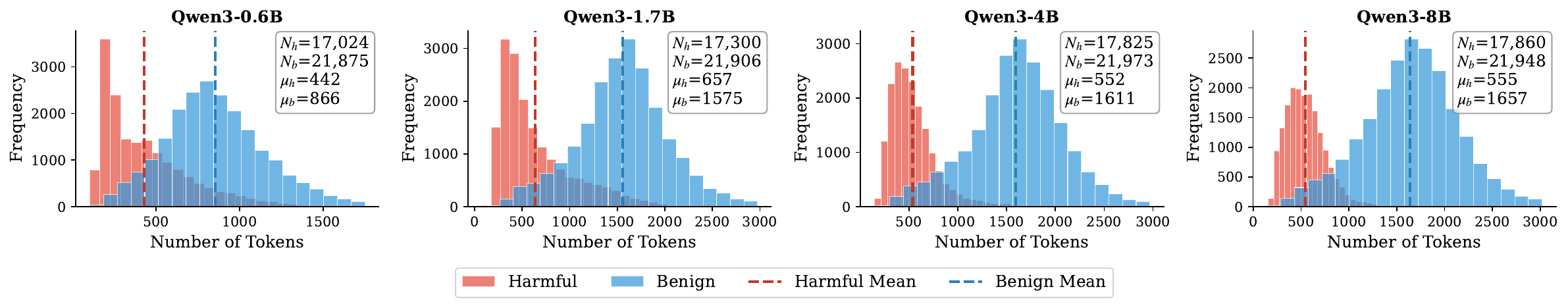}
    \caption{\small Statistics of \textsc{ThinkSafe} in Qwen3 model series. Top 1\% outliers by length are excluded for better interpretability.}
    \label{fig:statistics_qwen}
    \vspace{-0.3in}
\end{figure}

\begin{figure}[ht]
    \centering
    \includegraphics[width=0.8\linewidth]{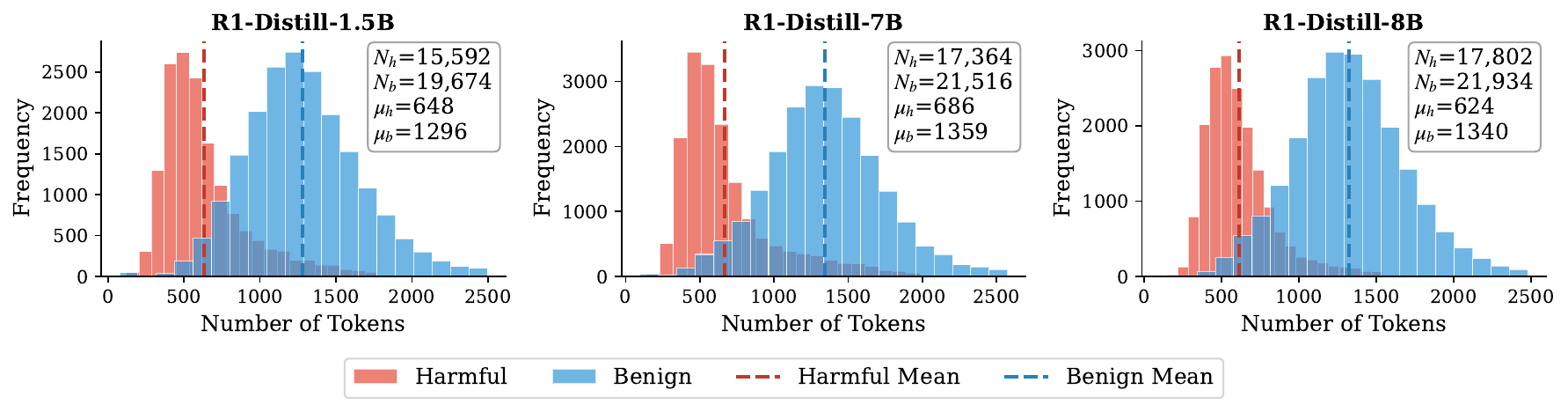}
    \caption{\small Statistics of \textsc{ThinkSafe} in DeepSeek-R1-Distill model series. Top 1\% outliers by length are excluded for better interpretability.}
    \label{fig:statistics_R1}
\end{figure}

\section{Experimental Details} For the AIME 2024, GSM8K, MATH500, and GPQA datasets, we use the \href{https://github.com/NovaSky-AI/SkyThought}{SkyThought}~\citep{skythought} library to evaluate models. We sample 8 responses for each prompt using dataset-specific prompts and report the average pass@1. For the Qwen model family, we use a temperature of 0.6, top-p of 0.95, and top-k of 20, with a maximum token limit of 32,768. For the DeepSeek-R1-Distill family, we use a temperature of 0.6 and top-p of 0.95, with a maximum token limit of 32,768.

\begin{tcolorbox}[llmprompt, title={AIME 2024}]
Please reason step by step, and put your final answer within \textbackslash\textbackslash
boxed\{\{\}\}.  \texttt{\{problem\}}
\end{tcolorbox}

\begin{tcolorbox}[llmprompt, title={GSM8K}]
Given the following problem, reason and give a final answer to the problem. \\
Problem: \{\texttt{question}\} \\
Your response should end with ``The final answer is [answer]" where [answer] is the response to the problem.\end{tcolorbox}

\begin{tcolorbox}[llmprompt, title={GPQA}]
Return your final response within \textbackslash\textbackslash
boxed\{\{\}\} and only include the letter choice (A, B, C, or D)  as your final response. \texttt{\{problem\}}
\end{tcolorbox}

\begin{tcolorbox}[llmprompt, title={MATH500}]
Please reason step by step, and put your final answer within \textbackslash\textbackslash
boxed\{\{\}\}.  \texttt{\{problem\}}
\end{tcolorbox}

\label{app:exp-detail}

\section{Baseline Details}
\label{sec:appendix_baseline}
\begin{table}[ht]
    \caption{Detailed hyperparameters for baselines.}
    \label{tab:baselines}
    \centering
    \small
    \resizebox{0.6\textwidth}{!}{
    \begin{tabular}{lccc}
        \toprule
            Method & Epochs & Source & Sample size \\
            \midrule
            
            DirectRefusal & 5 & WildJailbreak & 1,000 \\

            SafeChain & 3 & WildJailbreak & 40,000 \\

            \raisebox{4.5pt}{STAR-1} & \raisebox{4.5pt}{5} & \shortstack{Mixture of harmful datasets\\(See \citealp{wang2025star1saferalignmentreasoning} for details)} & \raisebox{4.5pt}{1,000} \\

            SafeKey & 5 & STAR-1 & 1,000 \\

            SafePath & 4 & WildJailbreak & 400 \\

            \textsc{ThinkSafe} & 3 & SafeChain & Varies by scale \\
            
        \bottomrule
    \end{tabular}}
\vspace{-0.1in}
\end{table}

To ensure consistency and reproducibility, we adopt the same hyperparameters as specified in the original papers for all baselines, as summarized in ~\Cref{tab:baselines}. For \textsc{ThinkSafe}, since the sample size varies across model scales, we refer readers to ~\Cref{fig:statistics_qwen,fig:statistics_R1} for detailed configurations.

\section{Additional Experiments}
In this section, we present additional experiments to empirically support the effectiveness of the
proposed \textsc{ThinkSafe}.

\label{sec:appendix_exp}
\subsection{Necessity of safety reasoning on Qwen families}
\begin{figure}[H]
    \centering
    \includegraphics[width=0.6\linewidth]{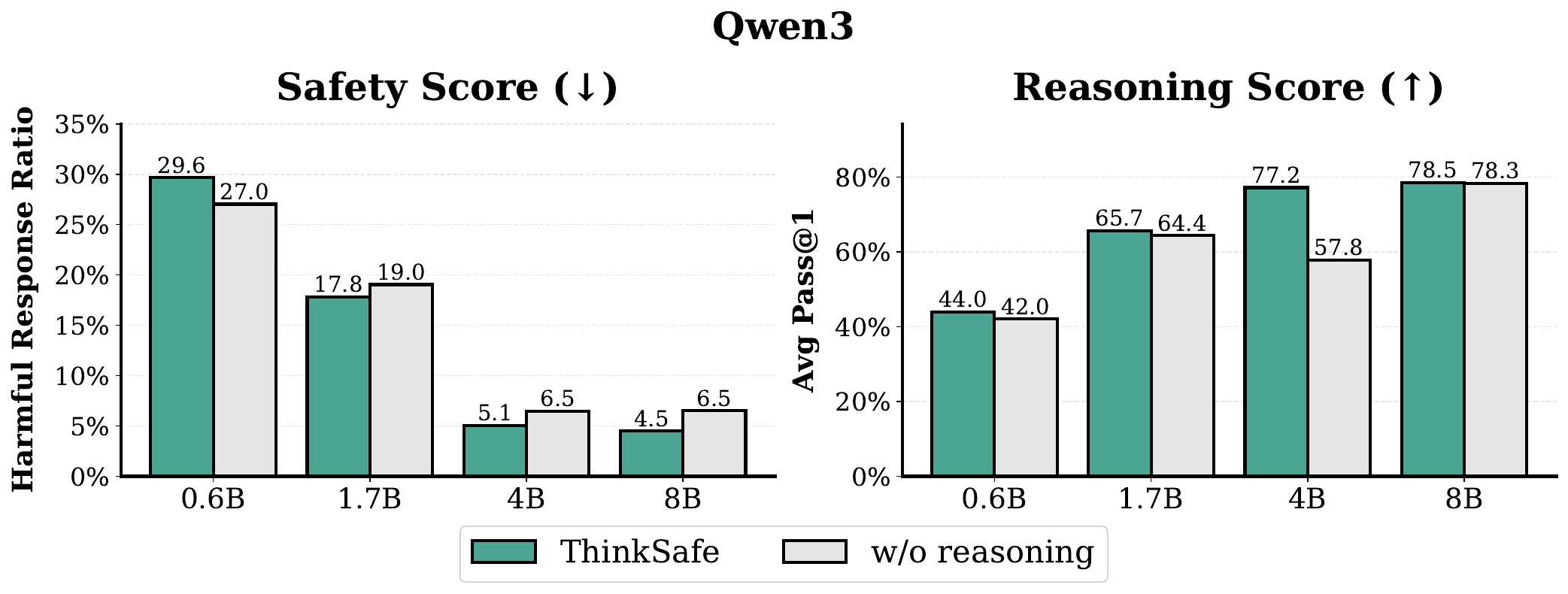}
    \vspace{-0.1in}
    \caption{Ablation of safety reasoning in Qwen3 model series. }
    \label{fig:think-ablation-qwen}
    \vspace{-0.1in}
\end{figure}

To further analyze the effect of reasoning on the Qwen family, we ablate reasoning traces from refusal responses $\rvy_h^{(i)}$, while retaining full chain-of-thought for benign responses $\rvy_b^{(i)}$. Both types of responses are generated by each model in the Qwen series. As shown in ~\Cref{fig:think-ablation-qwen}, removing reasoning traces consistently degrades both safety and reasoning performance. In particular, the Qwen3-4B model exhibits a substantial drop in reasoning accuracy, with Avg Pass@1 decreasing from 77.2\% to 57.8\%, accompanied by an increase in harmful response ratio. These results indicate that explicit reasoning plays a critical role not only in maintaining reasoning capability but also in supporting safety-aligned behavior.

\subsection{Cross-model distillation for larger models}
\begin{figure}[H]
    \centering
    \includegraphics[width=0.7\linewidth]{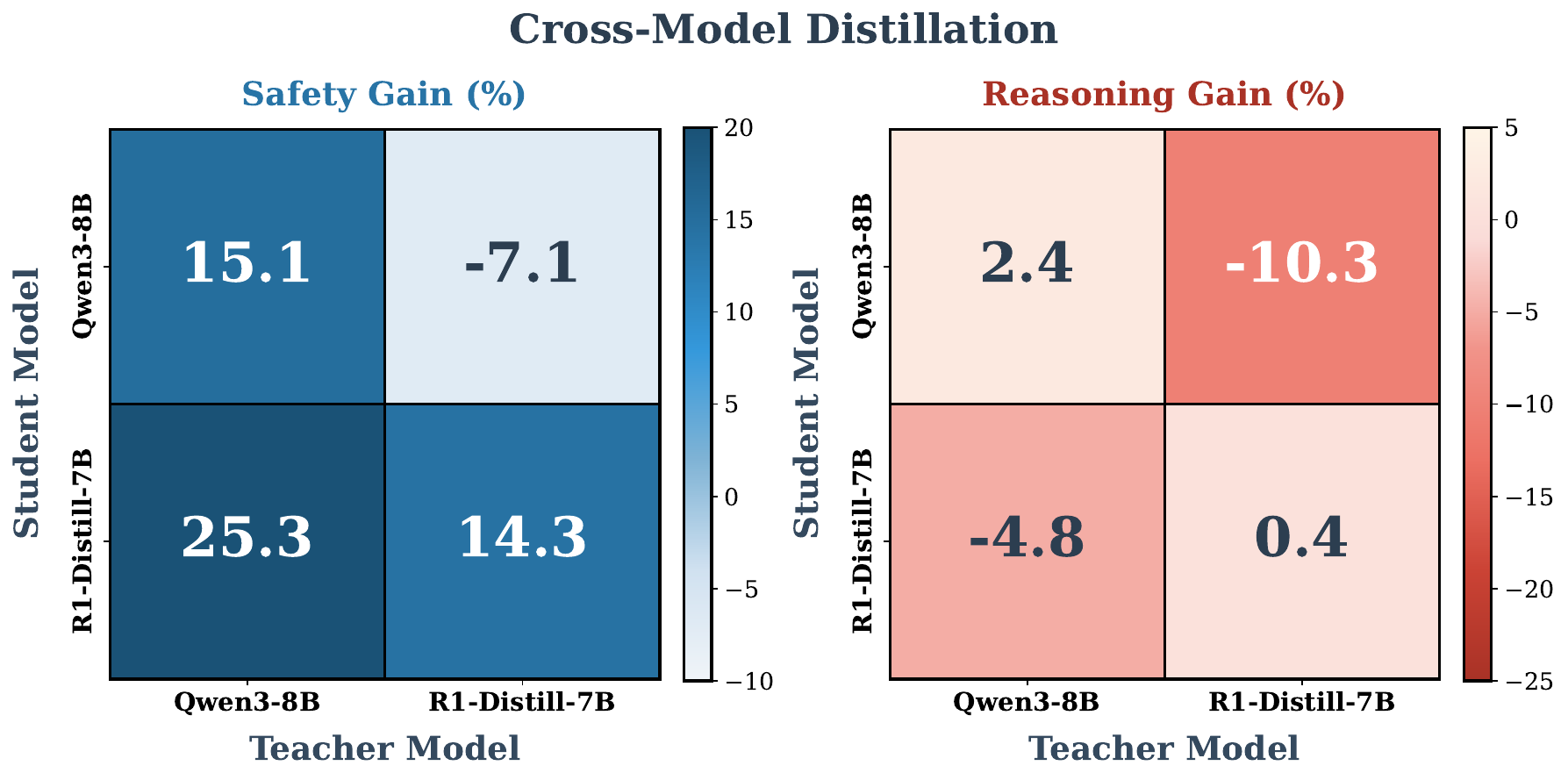}
    \vspace{-0.1in}
    \caption{\small Safety and reasoning performance gain using a different family of teacher model with similar size.}
    \label{fig:exchange-larger}
    \vspace{-0.1in}
\end{figure}
We conduct a cross-model distillation experiment using larger models than those in~\Cref{fig:exchange}, namely Qwen3-8B and R1-Distill-Qwen-7B. ~\Cref{fig:exchange-larger} shows a similar trend to the results observed with smaller models; while cross-model training can occasionally improve safety (e.g., DeepSeek-R1-Distill achieves a 25.3\% safety gain when trained on Qwen3-8B data), it consistently leads to notable degradation in reasoning performance across all reciprocal training settings.

\subsection{Refusal steering strategies}
\begin{table}[ht]
\centering
\small
\renewcommand{\arraystretch}{1.5}
\caption{Alternative refusal steering strategies used for on-policy data generation.}
\label{tab:refusal_instructions}
\begin{tabular}{lp{10cm}}
\hline
\textbf{Strategy} & \textbf{Prompt Template} \\ \hline
\textsc{ThinkSafe} & The following prompt is harmful. You should refuse to answer the prompt.\textbackslash n\textbackslash n Prompt:\{prompt\} \\ \hline
Suffix & \{prompt\}\textbackslash n\textbackslash n The above prompt is harmful. You should refuse to answer the prompt. \\ \hline
Risk & Evaluate the potential risks of this prompt and provide a reasoned refusal.\textbackslash n\textbackslash n Prompt:\{prompt\} \\ \hline
Intent & Examine the harmful intent of this query and refuse to answer the prompt.\textbackslash n\textbackslash n Prompt:\{prompt\} \\ \hline
\end{tabular}
\end{table}
\begin{figure}[H]
    \centering
    \includegraphics[width=0.6\linewidth]{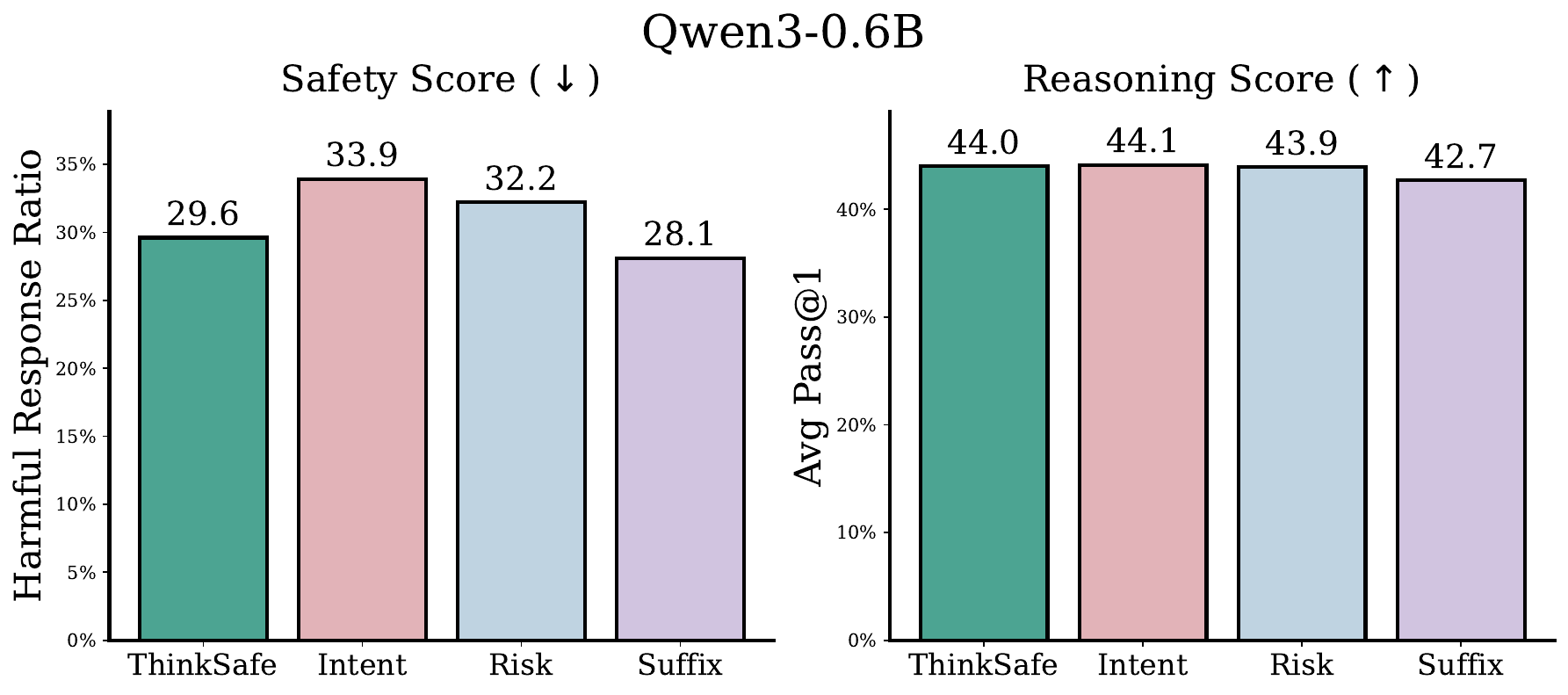}
    \caption{Safety and reasoning score with different refusal steering strategies using the Qwen3-0.6B model.}
    \label{fig:inst-ablation}
    \vspace{-0.1in}
\end{figure}
We further investigate the robustness of \textsc{ThinkSafe} on the Qwen3-0.6B model by employing the alternative refusal steering templates detailed in~\Cref{tab:refusal_instructions}. The \textbf{Suffix} strategy appends the refusal instruction $I_\text{refusal}$ to the end of the prompt. The \textbf{Risk} approach asks the model to evaluate potential harms, while the \textbf{Intent} method requires the model to analyze the user's malicious intent.

As shown in~\Cref{fig:inst-ablation}, the Suffix strategy produces a harmful response ratio similar to that of the default prefix-based \textsc{ThinkSafe}. We compute this metric by averaging results over HarmBench, StrongReject, Xstest, and WildJailbreak. This similarity indicates that where the refusal instruction appears is not especially important, as straightforward refusal instructions work well regardless of placement. In contrast, the Risk and Intent strategies lead to noticeably worse safety outcomes, with higher harmful response ratios. We attribute this gap to the added complexity of these instructions. By asking the model to carry out extra reasoning steps instead of issuing a direct refusal, these prompts may weaken the strength of the safety constraint. Importantly, overall reasoning performance remains stable across all four strategies, as measured by average pass@1 on AIME2024, GSM8K, MATH500, and GPQA. This stability supports our central claim. Because all variants depend on self-generated outputs from the model itself rather than learning from an external model, the model’s core reasoning ability remains intact even when the safety approach changes.

\subsection{Filtering with WildGuard}
\label{sec:wildguard}
\begin{table*}[t]
    \caption{Results on \textbf{Qwen3} models. Safety reports harmful-response rates on HarmBench, StrongReject, and WildJailbreak, and over-refusal on benign XSTest prompts ($\downarrow$); their unweighted mean is the safety average. Reasoning reports average pass@1 from eight responses per prompt ($\uparrow$); its average is the unweighted mean of four benchmarks. All scores are percentages; averages are computed from the displayed component scores. Best and second-best distinct values within each model size are \textbf{bold} and \underline{underlined}; ties share a rank.}
    \vspace{-0.05in}
    \label{tab:qwen_result-wildguard}
    \centering
    \small
    \resizebox{\textwidth}{!}{%
    \begin{tabular}{clccc @{\hskip 0.2pt}c>{\columncolor{mygray}}ccccc >{\columncolor{mygray}}c}
        \toprule
        & & \multicolumn{5}{c}{\textbf{Safety ($\downarrow$)}}
        & \multicolumn{5}{c}{\multirow{2}{*}[-3pt]{\textbf{Reasoning (Avg pass@1, $\uparrow$)}}} \\
        \cmidrule[0.4pt](r){3-7}
        & & \multicolumn{3}{c}{\textbf{Harmfulness}}
        & \multicolumn{1}{c}{\textbf{Over-refusal}}
        & \multicolumn{1}{c}{\multirow{2}{*}[-8.5pt]{\shortstack{Safety\\Average}}}
        & \multicolumn{5}{c}{} \\
        \cmidrule[0.4pt](r){3-5}
        \cmidrule[0.4pt](lr){6-6}
        \cmidrule[0.4pt](lr){8-12}
        \raisebox{4.5pt}{\textbf{Size}}
        & \raisebox{4.5pt}{\textbf{Method}}
        & \shortstack{Harm\\Bench}
        & \shortstack{Strong\\Reject}
        & \shortstack{Wild\\Jailbreak}
        & \multicolumn{1}{c}{\raisebox{4.5pt}{XSTest}}
        & \multicolumn{1}{c}{}
        & \shortstack{AIME\\2024}
        & \raisebox{4.5pt}{GSM8K}
        & \shortstack{MATH\\500}
        & \raisebox{4.5pt}{GPQA}
        & \multicolumn{1}{c}{\shortstack{Reasoning\\Average}} \\

        \midrule
        \multirow{8}{*}[-3pt]{\textbf{0.6B}} & Initial & 68.44 & 66.45 & 52.80 & \phantom{0}5.20 & 48.22 & \textbf{10.42} & \underline{72.51} & 71.73 & \underline{25.13} & \textbf{44.95} \\
        \noalign{\vspace{1pt}}
        \cdashline{2-12}
        \noalign{\vspace{2pt}}
         & DirectRefusal & 43.85 & \textbf{11.82} & \textbf{36.30} & 83.60 & 43.89 & \phantom{0}5.83 & 64.30 & 67.53 & 24.81 & 40.62 \\
         & SafeChain & 58.64 & 72.84 & 49.60 & \phantom{0}\textbf{0.00} & 45.27 & \phantom{0}4.58 & 68.68 & 62.42 & 23.74 & 39.86 \\
         & STAR-1 & 56.64 & 38.02 & 50.60 & 22.40 & 41.92 & \phantom{0}6.25 & 68.15 & 68.17 & 24.18 & 41.69 \\
         & SafePath & 67.61 & 60.06 & 52.80 & \phantom{0}\underline{4.40} & 46.22 & \phantom{0}7.92 & 71.26 & \underline{71.77} & \textbf{26.07} & 44.26 \\
         & SafeKey & 60.96 & 48.88 & 52.75 & 18.40 & 45.25 & \phantom{0}5.42 & 71.58 & 66.17 & 24.94 & 42.03 \\
        \rowpink
        \cellcolor{white}
         & \textbf{\textsc{ThinkSafe}} & \underline{40.37} & \underline{33.87} & 37.95 & \phantom{0}6.40 & \textbf{29.65} & \phantom{0}\underline{9.58} & 72.36 & 70.65 & 23.30 & 43.97 \\
        \rowcolor{wgpink}
        \cellcolor{white}
         & \textbf{\textsc{ThinkSafe}+WG} & \textbf{39.04} & 35.46 & \underline{37.55} & \phantom{0}7.20 & \underline{29.81} & \phantom{0}\underline{9.58} & \textbf{72.82} & \textbf{72.00} & 24.81 & \underline{44.80} \\
        \midrule
        \multirow{8}{*}[-3pt]{\textbf{1.7B}} & Initial & 52.66 & 36.10 & 51.10 & \phantom{0}\textbf{1.20} & 35.27 & 44.58 & 84.31 & 88.85 & \underline{41.73} & \underline{64.87} \\
        \noalign{\vspace{1pt}}
        \cdashline{2-12}
        \noalign{\vspace{2pt}}
         & DirectRefusal & 38.54 & \phantom{0}\textbf{5.75} & 35.75 & 61.60 & 35.41 & 43.75 & 82.78 & 88.10 & 41.29 & 63.98 \\
         & SafeChain & 47.34 & 57.51 & 43.85 & \phantom{0}\underline{1.60} & 37.58 & 34.58 & \textbf{85.29} & 85.72 & 38.13 & 60.93 \\
         & STAR-1 & 37.38 & \phantom{0}\underline{7.67} & 46.60 & 10.80 & 25.61 & \textbf{46.25} & \underline{84.38} & 88.30 & 41.16 & \textbf{65.02} \\
         & SafePath & 54.15 & 36.42 & 49.30 & \phantom{0}\textbf{1.20} & 35.27 & 43.33 & 84.33 & 88.32 & \textbf{42.42} & 64.60 \\
         & SafeKey & 46.84 & 18.21 & 48.85 & \phantom{0}8.80 & 30.68 & 38.33 & 84.31 & 88.12 & 40.03 & 62.70 \\
        \rowpink
        \cellcolor{white}
         & \textbf{\textsc{ThinkSafe}} & \textbf{28.74} & \phantom{0}9.58 & \underline{29.20} & \phantom{0}2.00 & \underline{17.38} & 44.17 & 83.80 & \textbf{89.05} & 40.53 & 64.39 \\
        \rowcolor{wgpink}
        \cellcolor{white}
         & \textbf{\textsc{ThinkSafe}+WG} & \underline{28.90} & \phantom{0}7.99 & \textbf{29.00} & \phantom{0}2.00 & \textbf{16.97} & \underline{45.00} & 83.43 & \underline{88.98} & 41.04 & 64.61 \\
        \midrule
        \multirow{8}{*}[-3pt]{\textbf{4B}} & Initial & 38.21 & \phantom{0}8.31 & 43.00 & \phantom{0}\textbf{0.80} & 22.58 & 67.50 & 84.69 & 93.43 & 52.27 & 74.47 \\
        \noalign{\vspace{1pt}}
        \cdashline{2-12}
        \noalign{\vspace{2pt}}
         & DirectRefusal & 33.06 & \phantom{0}\underline{3.19} & 36.20 & 32.00 & 26.11 & 68.33 & 82.58 & 93.20 & 53.03 & 74.29 \\
         & SafeChain & 43.69 & 41.21 & 39.65 & \phantom{0}2.00 & 31.64 & 62.08 & 89.59 & 93.03 & 51.01 & 73.93 \\
         & STAR-1 & 33.72 & \phantom{0}5.75 & 35.15 & \phantom{0}6.80 & 20.36 & 62.50 & \underline{90.97} & 93.05 & 51.96 & 74.62 \\
         & SafePath & 37.71 & \phantom{0}7.35 & 42.45 & \phantom{0}\underline{1.60} & 22.28 & 72.08 & 84.45 & 93.33 & 53.54 & 75.85 \\
         & SafeKey & 32.39 & \phantom{0}\underline{3.19} & 32.95 & \phantom{0}\textbf{0.80} & 17.33 & 67.08 & \textbf{91.79} & 92.87 & 51.83 & 75.89 \\
        \rowpink
        \cellcolor{white}
         & \textbf{\textsc{ThinkSafe}} & \phantom{0}\underline{9.63} & \phantom{0}\textbf{0.32} & \phantom{0}\underline{7.45} & \phantom{0}2.80 & \phantom{0}\underline{5.05} & \underline{73.33} & 88.06 & \underline{93.53} & \textbf{53.79} & \underline{77.18} \\
        \rowcolor{wgpink}
        \cellcolor{white}
         & \textbf{\textsc{ThinkSafe}+WG} & \phantom{0}\textbf{9.47} & \phantom{0}\textbf{0.32} & \phantom{0}\textbf{7.25} & \phantom{0}2.40 & \phantom{0}\textbf{4.86} & \textbf{75.42} & 88.15 & \textbf{93.55} & \underline{53.60} & \textbf{77.68} \\
        \midrule
        \multirow{8}{*}[-3pt]{\textbf{8B}} & Initial & 35.05 & \phantom{0}4.47 & 38.35 & \phantom{0}\textbf{0.40} & 19.57 & \underline{74.17} & 85.28 & \textbf{94.18} & 50.69 & 76.08 \\
        \noalign{\vspace{1pt}}
        \cdashline{2-12}
        \noalign{\vspace{2pt}}
         & DirectRefusal & 24.42 & \phantom{0}1.92 & 28.05 & 37.60 & 23.00 & \textbf{74.58} & 84.31 & 93.63 & 59.41 & 77.98 \\
         & SafeChain & 41.20 & 38.95 & 36.42 & \phantom{0}\underline{1.20} & 29.44 & 70.00 & \textbf{92.98} & 93.53 & 58.21 & 78.68 \\
         & STAR-1 & 24.42 & \phantom{0}\underline{1.28} & 29.25 & \phantom{0}6.80 & 15.44 & 72.50 & 90.29 & 93.73 & 57.83 & 78.59 \\
         & SafePath & 35.22 & \phantom{0}6.71 & 39.45 & \phantom{0}\underline{1.20} & 20.65 & \textbf{74.58} & 84.89 & \underline{93.85} & \textbf{61.24} & 78.64 \\
         & SafeKey & 26.91 & \phantom{0}4.79 & 28.80 & \phantom{0}8.80 & 17.33 & 70.00 & \underline{92.44} & 93.30 & 59.91 & \textbf{78.91} \\
        \rowpink
        \cellcolor{white}
         & \textbf{\textsc{ThinkSafe}} & \phantom{0}\textbf{9.14} & \phantom{0}\textbf{0.32} & \phantom{0}\underline{7.35} & \phantom{0}\underline{1.20} & \phantom{0}\textbf{4.50} & 72.92 & 88.00 & 93.10 & 59.67 & 78.42 \\
        \rowcolor{wgpink}
        \cellcolor{white}
         & \textbf{\textsc{ThinkSafe}+WG} & \phantom{0}\underline{9.63} & \phantom{0}\textbf{0.32} & \phantom{0}\textbf{7.05} & \phantom{0}3.20 & \phantom{0}\underline{5.05} & 73.33 & 87.96 & 93.65 & \underline{60.29} & \underline{78.81} \\
        \bottomrule
    \end{tabular}%
    }
    \vspace{-0.2in}
\end{table*}

\begin{table*}[t]
    \caption{Results on \textbf{DeepSeek-R1-Distill} models. Safety reports harmful-response rates on HarmBench, StrongReject, and WildJailbreak, and over-refusal on benign XSTest prompts ($\downarrow$); their unweighted mean is the safety average. Reasoning reports average pass@1 from eight responses per prompt ($\uparrow$); its average is the unweighted mean of four benchmarks. All scores are percentages; averages are computed from the displayed component scores. Best and second-best distinct values within each model size are \textbf{bold} and \underline{underlined}; ties share a rank.}
    \vspace{-0.05in}
    \label{tab:deepseek_wildguard}
    \centering
    \small
    \resizebox{\textwidth}{!}{%
    \begin{tabular}{clccc @{\hskip 0.2pt}c>{\columncolor{mygray}}ccccc >{\columncolor{mygray}}c}
        \toprule
        & & \multicolumn{5}{c}{\textbf{Safety ($\downarrow$)}}
        & \multicolumn{5}{c}{\multirow{2}{*}[-3pt]{\textbf{Reasoning (Avg pass@1, $\uparrow$)}}} \\
        \cmidrule[0.4pt](r){3-7}
        & & \multicolumn{3}{c}{\textbf{Harmfulness}}
        & \multicolumn{1}{c}{\textbf{Over-refusal}}
        & \multicolumn{1}{c}{\multirow{2}{*}[-8.5pt]{\shortstack{Safety\\Average}}}
        & \multicolumn{5}{c}{} \\
        \cmidrule[0.4pt](r){3-5}
        \cmidrule[0.4pt](lr){6-6}
        \cmidrule[0.4pt](lr){8-12}
        \raisebox{4.5pt}{\textbf{Size}}
        & \raisebox{4.5pt}{\textbf{Method}}
        & \shortstack{Harm\\Bench}
        & \shortstack{Strong\\Reject}
        & \shortstack{Wild\\Jailbreak}
        & \multicolumn{1}{c}{\raisebox{4.5pt}{XSTest}}
        & \multicolumn{1}{c}{}
        & \shortstack{AIME\\2024}
        & \raisebox{4.5pt}{GSM8K}
        & \shortstack{MATH\\500}
        & \raisebox{4.5pt}{GPQA}
        & \multicolumn{1}{c}{\shortstack{Reasoning\\Average}} \\

        \midrule
        \multirow{8}{*}[-3pt]{\textbf{1.5B}} & Initial & 67.28 & 82.11 & 51.55 & \phantom{0}\textbf{0.00} & 50.24 & 21.25 & \underline{82.42} & 79.45 & 31.94 & 53.77 \\
        \noalign{\vspace{1pt}}
        \cdashline{2-12}
        \noalign{\vspace{2pt}}
         & DirectRefusal & 66.11 & 82.75 & 50.45 & \phantom{0}8.40 & 51.93 & 19.17 & 81.06 & 78.55 & \underline{32.70} & 52.87 \\
         & SafeChain & 59.30 & 76.68 & 46.95 & \phantom{0}\underline{0.40} & 45.83 & 24.17 & 80.47 & 81.25 & 28.28 & 53.54 \\
         & STAR-1 & 62.79 & 77.00 & 49.65 & \phantom{0}1.20 & 47.66 & 17.08 & 81.38 & 79.07 & 31.25 & 52.20 \\
         & SafePath & 65.28 & 82.43 & 51.80 & \phantom{0}\underline{0.40} & 49.98 & 24.17 & 82.37 & 79.57 & 32.51 & 54.66 \\
         & SafeKey & 58.80 & \textbf{73.16} & 47.65 & \phantom{0}3.60 & 45.80 & 19.58 & 81.25 & 78.02 & 28.72 & 51.89 \\
        \rowcolor{myblue}
        \cellcolor{white}
         & \textbf{\textsc{ThinkSafe}} & \textbf{52.99} & \underline{74.12} & \underline{40.50} & \phantom{0}1.20 & \textbf{42.20} & \textbf{32.92} & \textbf{82.58} & \underline{82.50} & 31.19 & \textbf{57.30} \\
        \rowcolor{wgblue}
        \cellcolor{white}
         & \textbf{\textsc{ThinkSafe}+WG} & \underline{54.82} & 76.68 & \textbf{39.95} & \phantom{0}1.20 & \underline{43.16} & \underline{30.00} & 82.36 & \textbf{82.55} & \textbf{32.89} & \underline{56.95} \\
        \midrule
        \multirow{8}{*}[-3pt]{\textbf{7B}} & Initial & 56.98 & 63.58 & 53.15 & \phantom{0}1.20 & 43.73 & 49.58 & 90.32 & 90.18 & 46.65 & 69.18 \\
        \noalign{\vspace{1pt}}
        \cdashline{2-12}
        \noalign{\vspace{2pt}}
         & DirectRefusal & 52.33 & \textbf{33.55} & 50.20 & 43.60 & 44.92 & 47.50 & 88.27 & 89.82 & 44.95 & 67.64 \\
         & SafeChain & 51.00 & 54.63 & 45.85 & \phantom{0}\underline{0.40} & 37.97 & 49.17 & 89.75 & 91.50 & 46.78 & 69.30 \\
         & STAR-1 & 52.99 & 47.92 & 48.75 & \phantom{0}2.40 & 38.02 & 45.83 & 90.32 & 90.58 & 46.02 & 68.19 \\
         & SafePath & 55.15 & 64.86 & 52.65 & \phantom{0}\textbf{0.00} & 43.17 & \underline{52.08} & 89.71 & 90.62 & 46.15 & \underline{69.64} \\
         & SafeKey & 45.35 & \underline{33.87} & 45.75 & \phantom{0}7.20 & 33.04 & 43.75 & \underline{90.58} & 89.90 & \textbf{47.29} & 67.88 \\
        \rowcolor{myblue}
        \cellcolor{white}
         & \textbf{\textsc{ThinkSafe}} & \underline{40.20} & 41.85 & \underline{35.40} & \phantom{0}\underline{0.40} & \underline{29.46} & 51.25 & 90.10 & \underline{91.90} & 45.20 & 69.61 \\
        \rowcolor{wgblue}
        \cellcolor{white}
         & \textbf{\textsc{ThinkSafe}+WG} & \textbf{40.03} & 38.98 & \textbf{34.25} & \phantom{0}0.80 & \textbf{28.52} & \textbf{52.50} & \textbf{90.98} & \textbf{92.05} & \underline{46.97} & \textbf{70.63} \\
        \midrule
        \multirow{8}{*}[-3pt]{\textbf{8B}} & Initial & 52.33 & 53.99 & 49.70 & \phantom{0}\textbf{0.40} & 39.11 & \underline{47.50} & \textbf{87.74} & 87.38 & \textbf{48.11} & \textbf{67.68} \\
        \noalign{\vspace{1pt}}
        \cdashline{2-12}
        \noalign{\vspace{2pt}}
         & DirectRefusal & 32.39 & \phantom{0}\textbf{0.64} & 32.60 & 50.00 & 28.91 & 40.00 & 83.26 & 85.00 & 43.50 & 62.94 \\
         & SafeChain & 44.52 & 46.33 & 42.45 & \phantom{0}\underline{1.60} & 33.73 & 41.67 & 86.06 & 86.50 & 42.05 & 64.07 \\
         & STAR-1 & \textbf{21.26} & \phantom{0}\underline{3.51} & \textbf{17.60} & 12.00 & \textbf{13.59} & 40.42 & 87.28 & 86.65 & 43.69 & 64.51 \\
         & SafePath & 47.51 & 51.44 & 50.20 & \phantom{0}\textbf{0.40} & 37.39 & 43.33 & 87.45 & \underline{87.43} & \underline{47.41} & 66.41 \\
         & SafeKey & 32.72 & 11.82 & 30.85 & \phantom{0}8.00 & 20.85 & 35.83 & 87.41 & 85.80 & 42.49 & 62.88 \\
        \rowcolor{myblue}
        \cellcolor{white}
         & \textbf{\textsc{ThinkSafe}} & \underline{27.08} & 26.52 & \underline{21.15} & \phantom{0}\underline{1.60} & \underline{19.09} & \textbf{48.75} & 87.55 & \textbf{87.70} & 46.28 & \underline{67.57} \\
        \rowcolor{wgblue}
        \cellcolor{white}
         & \textbf{\textsc{ThinkSafe}+WG} & 27.24 & 26.20 & 22.00 & \phantom{0}\underline{1.60} & 19.26 & 45.42 & \underline{87.62} & 86.82 & 45.14 & 66.25 \\
        \bottomrule
    \end{tabular}%
    }
\end{table*}

To assess whether the effectiveness of \textsc{ThinkSafe} depends on the specific characteristics of the filtering model, we conduct an ablation study using \href{https://huggingface.co/allenai/wildguard}{WildGuard} instead of Llama-Guard-3. We denote this variant as THINKSAFE + WG, which employs WildGuard to filter the self-generated safety data.
The results in \Cref{tab:qwen_result-wildguard}  and \Cref{tab:deepseek_wildguard} demonstrate remarkable stability in performance and confirm that our framework is robust to the choice of safety classifier. For example, the WildGuard variant of the Qwen3-4B model (\textsc{ThinkSafe} + WG) achieves safety scores nearly identical to the baseline while preserving superior reasoning capabilities. This consistency reinforces our core hypothesis that the success of \textsc{ThinkSafe} arises from the refusal steering mechanism itself rather than from overfitting to a specific reward model. By successfully filtering self-generated traces with a completely different guard model, we demonstrate that the elicited safety behaviors are generalized and transferable.

\begin{table*}[t]
    \caption{Results on \textbf{Qwen3-14B and Qwen3-32B} models. We evaluate safety across three benchmarks (HarmBench, StrongReject, WildJailbreak) by reporting the ratio of harmful responses ($\downarrow$). Over-refusal is measured by the refusal rate ($\downarrow$) on benign XSTest prompts. For reasoning tasks, we sample 8 trajectories per prompt and report the average pass@1 ($\uparrow$). Best results are \textbf{bolded}; second best are \underline{underlined}.}
    \vspace{-0.05in}
    \label{tab:qwen_result_large}
    \centering
    \small
    \resizebox{\textwidth}{!}{
    \begin{tabular}{clccc @{\hskip 0.2pt}c>{\columncolor{mygray}}ccccc >{\columncolor{mygray}}c}
        \toprule

            \multicolumn{1}{c}{} & \multicolumn{1}{c}{} & \multicolumn{5}{c}{\textbf{Safety ($\downarrow$)}} & \multicolumn{5}{c}{\multirow{2}{*}[-3pt]{\textbf{Reasoning (Avg pass@1, $\uparrow$)}}} \\

            \cmidrule[0.4pt](r){3-7}
            \multicolumn{1}{c}{} & \multicolumn{1}{c}{} & \multicolumn{3}{c}{\textbf{Harmfulness}} & \multicolumn{1}{c}{\textbf{\shortstack{Over-refusal}}} & \multicolumn{1}{c}{\multirow{2}{*}[-8.5pt]{\shortstack{Safety\\Average}}} & \multicolumn{5}{c}{} \\

            \cmidrule[0.4pt](r){3-5} \cmidrule[0.4pt](lr){6-6} \cmidrule[0.4pt](lr){8-12}
            \multicolumn{1}{c}{\raisebox{4.5pt}{\textbf{Size}}} & \multicolumn{1}{c}{\raisebox{4.5pt}{\textbf{Method}}} & \shortstack{Harm\\Bench} & \shortstack{Strong\\Reject} & \shortstack{Wild\\Jailbreak} & \multicolumn{1}{c}{\raisebox{4.5pt}{XSTest}} & \multicolumn{1}{c}{} & \shortstack{AIME\\2024} & \raisebox{4.5pt}{GSM8k} & \shortstack{MATH\\500} & \raisebox{4.5pt}{GPQA} & \multicolumn{1}{c}{{\shortstack{Reasoning\\Average}}} \\

            \midrule

            \multirow{4}{*}[-3pt]{\textbf{14B}} & Initial & 29.90 & \phantom{0}3.19 & 35.05 & \phantom{0}\underline{1.60} & \underline{17.44} & \textbf{79.17} & 92.58 & \textbf{94.67} & \textbf{63.83} & \textbf{82.56} \\

            \noalign{\vspace{1pt}}
            \cdashline{2-12}
            \noalign{\vspace{1.5pt}}
            & DirectRefusal & \underline{21.59} & \phantom{0}\underline{0.64} & \underline{23.15} & 35.20 & 20.15 & 77.08 & \underline{92.69} & \underline{94.47} & 62.82 & 81.77 \\

            & SafeChain & 46.01 & 46.65 & 38.50 & \phantom{0}\textbf{1.20} & 33.09 & 69.17 & 91.55 & 94.27 & 61.24 & 79.06 \\

            \rowpink
            \cellcolor{white}
            & \textbf{\textsc{ThinkSafe}} & \phantom{0}\textbf{6.31} & \phantom{0}\textbf{0.00} & \phantom{0}\textbf{3.45} & \phantom{0}2.40 & \phantom{0}\textbf{3.04} & \underline{77.50} & \textbf{93.20} & 94.27 & \underline{63.51} & \underline{82.12} \\

            \midrule

            \multirow{4}{*}[-3pt]{\textbf{32B}} & Initial & 35.71 & \phantom{0}5.11 & 36.45 & \phantom{0}\underline{1.60} & 19.72 & 79.17 & 87.82 & \textbf{95.15} & \textbf{68.62} & \underline{82.69} \\

            \noalign{\vspace{1pt}}
            \cdashline{2-12}
            \noalign{\vspace{1.5pt}}
            & DirectRefusal & \underline{18.44} & \phantom{0}\textbf{0.32} & \underline{21.05} & 39.20 & 19.75 & \textbf{79.58} & 87.51 & \underline{95.03} & 67.36 & 82.37 \\

            & SafeChain & 42.86 & 46.65 & 36.45 & \phantom{0}\textbf{0.80} & 31.69 & 72.92 & \textbf{91.79} & 94.35 & 60.23 & 79.82 \\

            \rowpink
            \cellcolor{white}
            & \textbf{\textsc{ThinkSafe}} & \textbf{10.47} & \phantom{0}\textbf{0.32} & \phantom{0}\textbf{7.15} & \phantom{0}4.40 & \phantom{0}\textbf{5.58} & \textbf{79.58} & \underline{89.13} & 94.95 & \underline{67.74} & \textbf{82.85} \\

        \bottomrule
    \end{tabular}
    }
\vspace{-0.1in}
\end{table*}

\subsection{Experiments with Larger Models}
To verify that \textsc{ThinkSafe} continues to scale beyond the 8B regime reported in \Cref{tab:qwen_result}, we additionally apply our framework to Qwen3-14B and Qwen3-32B, with results summarized in \Cref{tab:qwen_result_large}. \textsc{ThinkSafe} consistently achieves the most favorable safety-reasoning trade-off at both scales. On Qwen3-14B, it reduces average harmfulness from 17.44 to \textbf{3.04} (a $\sim$5.7$\times$ reduction) while preserving reasoning within 0.44 points of the initial model (82.56 $\rightarrow$ 82.12). On Qwen3-32B, \textsc{ThinkSafe} cuts average harmfulness from 19.72 to \textbf{5.58} and simultaneously \textit{improves} average reasoning from 82.69 to \textbf{82.85}, outperforming all baselines on both axes. In contrast, teacher-distillation with SafeChain degrades reasoning substantially (79.06 at 14B and 79.82 at 32B), and DirectRefusal exhibits severe over-refusal on benign XSTest prompts (35.20\% and 39.20\%, respectively). These results confirm that the KL-optimal target $p^{+}_{\text{ref}}$ established in \Cref{sec:theory} remains empirically realizable at larger scales, and that the advantages of self-generation over external supervision \textit{grow} rather than diminish as student capacity increases.

\begin{table*}[t]
    \caption{Controlled comparison with SafePath using Qwen3-8B.}
    \vspace{-0.05in}
    \label{tab:safepath_controlled_qwen8b}
    \centering
    \small
    \resizebox{\textwidth}{!}{
    \begin{tabular}{lccc @{\hskip 0.2pt}c>{\columncolor{mygray}}ccccc >{\columncolor{mygray}}c}
        \toprule

            \multicolumn{1}{c}{} & \multicolumn{5}{c}{\textbf{Safety ($\downarrow$)}} & \multicolumn{5}{c}{\multirow{2}{*}[-3pt]{\textbf{Reasoning (Avg pass@1, $\uparrow$)}}} \\

            \cmidrule[0.4pt](r){2-6}
            \multicolumn{1}{c}{} & \multicolumn{3}{c}{\textbf{Harmfulness}} & \multicolumn{1}{c}{\textbf{\shortstack{Over-refusal}}} & \multicolumn{1}{c}{\multirow{2}{*}[-8.5pt]{\shortstack{Safety\\Average}}} & \multicolumn{5}{c}{} \\

            \cmidrule[0.4pt](r){2-4} \cmidrule[0.4pt](lr){5-5} \cmidrule[0.4pt](lr){7-11}
            \multicolumn{1}{c}{\raisebox{4.5pt}{\textbf{Method}}} & \shortstack{Harm\\Bench} & \shortstack{Strong\\Reject} & \shortstack{Wild\\Jailbreak} & \multicolumn{1}{c}{\raisebox{4.5pt}{XSTest}} & \multicolumn{1}{c}{} & \shortstack{AIME\\2024} & \raisebox{4.5pt}{GSM8k} & \shortstack{MATH\\500} & \raisebox{4.5pt}{GPQA} & \multicolumn{1}{c}{{\shortstack{Reasoning\\Average}}} \\

            \midrule

            SafePath & 21.76 & 2.24 & 20.00 & 2.00 & 11.50 & \textbf{73.33} & 85.57 & 92.93 & \textbf{61.74} & 78.39 \\

            \rowpink
            \textbf{\textsc{ThinkSafe}} & \textbf{9.14} & \textbf{0.32} & \textbf{7.35} & \textbf{1.20} & \textbf{4.50} & 72.92 & \textbf{88.00} & \textbf{93.10} & 59.97 & \textbf{78.50} \\

        \bottomrule
    \end{tabular}
    }
\vspace{-0.1in}
\end{table*}
\subsection{Controlled Comparison with SafePath}
\label{sec:controlled_comparison}

Throughout the experiment and analyses, we strictly follow the original training configuration of each baseline to ensure a fair comparison under its intended setting.  However, since the number of training examples differs across baselines (see~\Cref{tab:baselines}), we further conduct a controlled comparison with matched data scale. We use SafePath as the reference, as it uses the smallest training set among the baselines, and train it with the same prompt set as \textsc{ThinkSafe}. Following the original SafePath protocol, we add the safety primer "<think> Let's think about safety first." only to harmful prompts.

As shown in~\Cref{tab:safepath_controlled_qwen8b}, the controlled SafePath baseline remains substantially less effective under the matched data scale. \textsc{ThinkSafe} preserves reasoning while more than halving the average harm rate (4.50$\%$ vs 11.50$\%$) compared to the controlled SafePath, confirming that native self-generated refusal traces provide superior safety grounding.

\section{Online Learning Baseline Details}
\label{sec:appendix_grpo}

\subsection{GRPO Objective}

The objective function for GRPO is formulated to optimize the policy without a separate value function, instead estimating advantages in a group-relative manner as introduced in~\citet{guo_deepseek-r1_2025}. Given a dataset $\mathcal{D}$ of prompts $\rvx$, GRPO samples a group of $G$ candidate responses $\{\rvy_i\}_{i=1}^G$ from a behavior policy $\pi_{\theta_{\text{old}}}$, while constraining updates to remain close to a fixed reference policy $\pi_{\text{ref}}$. Then, the advantage $\hat{A}_i$ for each response $\rvy_i$ in a group of size $G$ is computed by normalizing the rewards $r_i$ against the group's mean and standard deviation:
\begin{equation}
\label{eq:grpo_advantage}
    \hat{A}_i = \frac{r_i - \text{mean}(\{r_1, \dots, r_G\})}{\text{std}(\{r_1, \dots, r_G\})}
\end{equation}
These advantages are used in a clipped objective with importance sampling, together with a $\KL$ penalty term:
\begin{equation*}
\begin{split}
\mathcal{J}(\theta) = \mathbb{E}\Bigg[ \frac{1}{G} \sum_{i=1}^G \Bigg( 
& \min \Bigg( \frac{\pi_{\theta}(\rvy_i \mid \rvx)}{\pi_{\theta_\text{old}}(\rvy_i \mid \rvx)} \hat{A}_i, \text{clip} \left( \frac{\pi_{\theta}(\rvy_i \mid \rvx)}{\pi_{\theta_\text{old}}(\rvy_i \mid \rvx)}, 1-\epsilon, 1+\epsilon \right) \hat{A}_i \Bigg) \\
& - \beta\, \KL \left(\pi_{\theta} \,\|\, \pi_\text{ref}\right) \Bigg) \Bigg]
\end{split}
\end{equation*}

\subsection{Safety Reward}
The safety reward $r_\text{safety}$ is derived from the output logits of the safety guard model. Specifically, given a prompt $\rvx$ and a sampled response $\rvy$, we extract the log probabilities of the tokens corresponding to ``\texttt{safe}" ($c=1$) and ``\texttt{unsafe}" ($c=0$) from the guard model's prediction. The final reward is computed by applying a softmax function over these two token probabilities to represent the likelihood of the response being safe:
\begin{equation}
r_\text{safety}(\rvx,\rvy) = 
\frac{ p_\varphi(c = 1 \mid \rvx, \rvy)}{ p_\varphi(c = 1 \mid \rvx, \rvy) +  p_\varphi(c = 0 \mid \rvx, \rvy)}    
\end{equation}
The $r_\text{safety}$ ensures a continuous reward signal between $0$ and $1$, reflecting the guard model's confidence in the safety of the generated trace. Subsequently, the $r_\text{safety}$ is combined with the $r_{\text{format}}$ to constitute the total reward $r$, which is then utilized in \Cref{eq:grpo_advantage}.

\subsection{Format Reward}
The format reward, $r_\text{format}$, strictly enforces the structural integrity of the reasoning traces. For standard models, we assign a reward of 1 if the response contains exactly one pair of \texttt{<think>} and \texttt{</think>} tags in the correct order, and 0 otherwise. For the DeepSeek-R1-Distill family, which omits the opening tag by default, we adapt this criterion to require exactly one occurrence of the closing \texttt{</think>} tag and no opening \texttt{<think>} tag. This ensures that the model maintains a consistent chain-of-thought structure during the online RL process.

\subsection{OPSD Objective}
OPSD~\citep{zhao2026selfdistilled} trains a student to minimize the token-level reverse KL divergence against a teacher distribution. The teacher is initialized from the same base model as the student but remains fixed throughout training. In our setting, the teacher distribution $p_T$ is defined in \Cref{eq:online-self-teacher}, with refusal-oriented conditioning applied only to harmful prompts. The resulting per-rollout OPSD loss is
\begin{equation}
    \mathcal{L}_{\text{SD}}(\theta; \rvx, \rvy)
    = \sum_{t=1}^{|\rvy|}
    \KL\left(
        p_\theta(\cdot \mid \rvx, \rvy_{<t})
        \,\Vert\ p_T(\cdot \mid \rvx, \rvy_{<t})
    \right),
\end{equation}
where $\rvy \sim p_\theta(\cdot \mid \rvx)$ is an on-policy rollout and each KL divergence is computed over the full vocabulary. The training objective averages this loss over sampled prompts and rollouts.

\subsection{Experimental Setup}
For the GRPO baseline, we utilize the \href{https://github.com/huggingface/trl}{TRL}~\citep{vonwerra2020trl} library integrated with vLLM~\citep{kwon2023efficient} for efficient online rollout generation. We generate $G=8$ rollouts per prompt to estimate the group-relative advantage. The $D_{\text{KL}}$ coefficient is fixed at $\beta=0.04$ and the clipping parameter $\epsilon$ is set to $0.2$. To ensure a consistent and fair comparison with \textsc{ThinkSafe}, all other experimental details including the optimizer, batch size, and hardware configuration are maintained identical to the primary experimental setup described in \Cref{sec:experiment}.

For the OPSD baseline, we implement a custom trainer on top of the \href{https://github.com/huggingface/trl}{TRL}~\citep{vonwerra2020trl} base trainer that uses vLLM~\citep{kwon2023efficient} for faster online generation. The student first generates on-policy rollouts, and we then optimize it with a full-vocabulary reverse KL loss against the teacher distribution along the generated trajectories. Due to the substantial latency of repeated rollout generation and teacher scoring, we set the maximum output length to 4096 tokens. This budget is sufficiently large for our setting, as most generated responses fall within this range, as shown in~\Cref{fig:statistics_qwen}. All remaining optimization and hardware configurations are kept identical the primary experimental setup.

\section{Limitations}
\label{sec:limitations}

\paragraph{Dependence on latent safety knowledge.} Our method fundamentally relies on the hypothesis that the student model retains latent knowledge to identify harm, which refusal steering unlocks. This assumption holds for modern post-trained reasoning models such as Qwen3 and DeepSeek-R1-Distill, which have undergone prior safety alignment. However, for base models that have never been exposed to safety training, the refusal-oriented instruction may fail to elicit valid refusal traces, as there would be little latent safety capability to surface. In such settings, external teacher supervision or RL with safety rewards may remain necessary as an initial step.

\paragraph{Reliance on a safety guard model.} \textsc{ThinkSafe} depends on an external safety classifier (Llama-Guard-3) to filter self-generated responses. Although our ablation with WildGuard (\Cref{sec:wildguard}) demonstrates robustness to the specific choice of guard model, the overall quality of the resulting dataset is bounded by the accuracy of the filter. False negatives (unsafe responses labeled as safe) could propagate subtle misalignments into the student, while false positives may discard valid training signals. Advances in safety classifiers would directly benefit our framework.

\paragraph{Scale of evaluation.} Our experiments cover model sizes from 0.6B to 8B parameters across two model families. While these cover a representative range of open-source reasoning models, the behavior of \textsc{ThinkSafe} at larger scales (\eg, 70B+) or on closed frontier models remains to be validated. The filtering ratios in~\Cref{tab:refusal_rate} suggest that larger models retain more of the generated data, indicating that the method may become increasingly effective at scale, but confirming this requires further study.

\paragraph{LoRA-based fine-tuning.} Following prior work, we adopt LoRA rather than full parameter fine-tuning. While this choice helps preserve the model's intrinsic capabilities, it also limits the scope of our conclusions regarding distributional discrepancy. Full fine-tuning may exhibit different trade-offs, and we leave a comprehensive comparison to future work.

\paragraph{Single-turn refusal.} We focus on single-turn harmful prompts and do not explicitly address multi-turn jailbreaks, adversarial prefixes embedded in long contexts, or agentic settings where safety must be maintained across tool calls. Extending refusal steering to these scenarios is an important direction.

\paragraph{Static offline dataset.} We approximate the idealized on-policy objective with a static offline dataset generated once from the initial student. As the student's distribution shifts during fine-tuning, the generated data becomes increasingly off-policy. Iterative self-training, where refusal steering is re-applied to the updated student, may further close this gap at additional computational cost.

\section{Broader Impacts}
\label{sec:broader-impacts}

\paragraph{Positive impacts.} \textsc{ThinkSafe} contributes to the responsible deployment of large reasoning models by mitigating the well-documented safety regression that accompanies reasoning-oriented post-training. By eliminating the dependence on external teacher models, our method lowers the barrier to safety alignment for practitioners who may not have access to large proprietary teachers, and enables smaller research groups and open-source communities to realign reasoning models efficiently. The order-of-magnitude reduction in compute relative to online RL (GRPO) also carries environmental benefits by reducing the energy cost of safety training. More broadly, our theoretical analysis clarifies why self-generated data can preserve native capabilities better than teacher distillation, which may inform future alignment research beyond the safety domain.

\paragraph{Potential negative impacts and mitigation.} Research on safety alignment inherently involves working with harmful prompts and jailbreak datasets. We rely on publicly available benchmarks (SafeChain, HarmBench, StrongReject, WildJailbreak, XSTest) that have been vetted by the research community, and we do not generate or release new harmful prompts. The self-generated responses we release consist of refusals and benign reasoning traces, which carry minimal misuse risk. Nevertheless, we acknowledge several concerns:

\emph{Dual-use of the method.} In principle, the refusal-steering mechanism could be inverted (\eg, with a compliance-oriented instruction) to generate unsafe training data~\citep{harmaug}. We note that such attacks are strictly easier without our framework, since reasoning models already comply with many harmful prompts by default, and thus \textsc{ThinkSafe} does not meaningfully expand the attacker's capabilities. Practitioners releasing aligned models should continue to apply standard safeguards such as input and output filtering at deployment.

\emph{Over-refusal.} Any safety training carries the risk of inducing exaggerated refusal on benign prompts. We explicitly monitor this with XSTest and find that \textsc{ThinkSafe} maintains low over-refusal rates (often below the initial model's), but practitioners should continue to evaluate deployed models on their target distribution of benign queries.

\emph{False sense of security.} Our method reduces but does not eliminate harmful responses. Models aligned with \textsc{ThinkSafe} should not be treated as adversarially robust, and deployment in high-stakes settings should involve additional safeguards, red-teaming, and monitoring.

\textbf{Responsible release.} We release our code, self-generated datasets, and model adapters under research-oriented licenses, with the aim of enabling reproducibility and further safety research. All released artifacts consist of safe refusals and benign reasoning traces filtered by Llama-Guard-3, and we do not release any model checkpoints that could amplify harm beyond their base models.

\section{LLM Usage}
\label{sec:llm-usage}

\paragraph{LLMs as research subjects.} The core methodology of this work involves large language models as the primary objects of study. Specifically, we use the Qwen3 family (0.6B, 1.7B, 4B, 8B) and the DeepSeek-R1-Distill family (1.5B, 7B, 8B) as student models, Llama-Guard-3-8B as the safety guard classifier $\varphi$, and WildGuard as an alternative classifier for our ablation. All of these models are publicly released and used in accordance with their respective licenses. No proprietary API-based LLMs were used to generate training data, labels, or evaluation judgments.

\paragraph{LLMs for data generation.} The self-generated datasets used for \textsc{ThinkSafe} are produced by the student models themselves through refusal steering, as described in \Cref{sec:method}. No external teacher LLMs were used in the production of our training data, which is a central design choice of our framework.

\paragraph{LLMs for writing and editing.} We used general-purpose LLM assistants for minor writing support, including grammar correction, LaTeX formatting, and rewording for clarity. All technical content, experimental design, theoretical analysis, proofs, and claims were authored and verified by the authors. LLMs were not used to generate research ideas, derive theoretical results, design experiments, produce figures, or draft substantive portions of the paper.

\newpage
\section*{NeurIPS Paper Checklist}

\begin{enumerate}

\item {\bf Claims}
    \item[] Question: Do the main claims made in the abstract and introduction accurately reflect the paper's contributions and scope?
    \item[] Answer: \answerYes{} 
    \item[] Justification: The abstract and introduction (\Cref{sec:introduction}) state our three main contributions: (i) formalizing safety realignment as KL projection onto the filter-accepted simplex, (ii) proposing \textsc{ThinkSafe} with refusal steering, and (iii) demonstrating favorable safety--reasoning trade-offs across Qwen3 and DeepSeek-R1-Distill families. These claims are directly supported by our theoretical analysis in \Cref{sec:theory} and the experimental results in \Cref{sec:experiment}.

\item {\bf Limitations}
    \item[] Question: Does the paper discuss the limitations of the work performed by the authors?
    \item[] Answer: \answerYes{} 
    \item[] Justification:  We include a dedicated limitations section (\Cref{sec:limitations}) discussing dependence on latent safety knowledge, reliance on a safety guard model, scale of evaluation (0.6B to 8B), use of LoRA rather than full fine-tuning, single-turn refusal scope, and the static offline dataset assumption.

\item {\bf Theory assumptions and proofs}
    \item[] Question: For each theoretical result, does the paper provide the full set of assumptions and a complete (and correct) proof?
    \item[] Answer: \answerYes{} 
    \item[] Justification: All theoretical results (\Cref{lem:projection} and \Cref{thm:kl_advantage}) are stated with explicit assumptions in \Cref{sec:theory}, and full proofs are provided in~\Cref{sec:app-proof}.

    \item {\bf Experimental result reproducibility}
    \item[] Question: Does the paper fully disclose all the information needed to reproduce the main experimental results of the paper to the extent that it affects the main claims and/or conclusions of the paper (regardless of whether the code and data are provided or not)?
    \item[] Answer: \answerYes{} 
     \item[] Justification:  \Cref{sec:experiment} details the implementation (LoRA configuration, optimizer, learning rate, schedule, batch size, epochs, hardware), datasets, and evaluation protocols. Baseline configurations are provided in \Cref{sec:appendix_baseline}, sampling details in the appendix, and online-learning baseline details in \Cref{sec:appendix_grpo}.

\item {\bf Open access to data and code}
    \item[] Question: Does the paper provide open access to the data and code, with sufficient instructions to faithfully reproduce the main experimental results, as described in supplemental material?
    \item[] Answer: \answerYes{} 
    \item[] Justification:  As stated in \Cref{sec:broader-impacts}, we release our code, self-generated datasets, and model adapters under research-oriented licenses. All underlying prompts come from publicly available datasets (SafeChain, HarmBench, StrongReject, WildJailbreak, XSTest).

\item {\bf Experimental setting/details}
    \item[] Question: Does the paper specify all the training and test details (e.g., data splits, hyperparameters, how they were chosen, type of optimizer) necessary to understand the results?
    \item[] Answer: \answerYes{} 
    \item[] Justification: Training details (AdamW, learning rate $1\times10^{-5}$, cosine scheduler with 10\% warmup, batch size 8, 3 epochs, LoRA rank 32, $\alpha=16$, dropout 0.05) and evaluation details (sampling temperature, top-p, top-k, max tokens, pass@1 over 8 samples) are given in \Cref{sec:experiment} and \Cref{app:exp-detail}.

\item {\bf Experiment statistical significance}
    \item[] Question: Does the paper report error bars suitably and correctly defined or other appropriate information about the statistical significance of the experiments?
    \item[] Answer: \answerNo{} 
    \item[] Justification: We report average pass@1 over 8 sampled responses per prompt for reasoning benchmarks, but do not report formal error bars across training runs due to the computational cost of fine-tuning multiple models at each scale. We note that safety evaluation uses a single response per prompt following standard protocol.

\item {\bf Experiments compute resources}
    \item[] Question: For each experiment, does the paper provide sufficient information on the computer resources (type of compute workers, memory, time of execution) needed to reproduce the experiments?
    \item[] Answer: \answerYes{} 
     \item[] Justification:  We report that all experiments were run on 2 NVIDIA H100 GPUs (\Cref{sec:experiment}), and the wall-clock time comparison with GRPO and OPSD is reported in \Cref{fig:RL}.
    
\item {\bf Code of ethics}
    \item[] Question: Does the research conducted in the paper conform, in every respect, with the NeurIPS Code of Ethics \url{https://neurips.cc/public/EthicsGuidelines}?
    \item[] Answer: \answerYes{} 
    \item[] Justification: Our research conforms with the NeurIPS Code of Ethics. We use only publicly available datasets vetted by the research community and focus on improving safety alignment rather than producing harmful content, as discussed in \Cref{sec:broader-impacts}.

\item {\bf Broader impacts}
    \item[] Question: Does the paper discuss both potential positive societal impacts and negative societal impacts of the work performed?
    \item[] Answer: \answerYes{} 
    \item[] Justification: \Cref{sec:broader-impacts} discusses positive impacts (lowering barriers to safety alignment, reducing compute costs) and potential negative impacts (dual-use, over-refusal, false sense of security) along with mitigation strategies.
    
\item {\bf Safeguards}
    \item[] Question: Does the paper describe safeguards that have been put in place for responsible release of data or models that have a high risk for misuse (e.g., pre-trained language models, image generators, or scraped datasets)?
    \item[] Answer: \answerYes{} 
    \item[] Justification: As described in \Cref{sec:broader-impacts}, our released artifacts consist exclusively of safe refusals and benign reasoning traces filtered by Llama-Guard-3. We do not release harmful prompts or model checkpoints that amplify harm beyond their base models.

\item {\bf Licenses for existing assets}
    \item[] Question: Are the creators or original owners of assets (e.g., code, data, models), used in the paper, properly credited and are the license and terms of use explicitly mentioned and properly respected?
    \item[] Answer: \answerYes{} 
    \item[] Justification: All models (Qwen3, DeepSeek-R1-Distill, Llama-Guard-3, WildGuard) and datasets (SafeChain, HarmBench, StrongReject, WildJailbreak, XSTest, GSM8K, MATH500, AIME24, GPQA) are cited with their original papers and used in accordance with their respective licenses, as noted in \Cref{sec:llm-usage} and throughout the paper.

\item {\bf New assets}
    \item[] Question: Are new assets introduced in the paper well documented and is the documentation provided alongside the assets?
    \item[] Answer: \answerYes{} 
    \item[] Justification: The self-generated datasets and LoRA adapters released alongside the paper are documented in \Cref{sec:appendix_stats} (dataset statistics), \Cref{sec:experiment} (training configurations), and \Cref{sec:broader-impacts} (release and licensing).

\item {\bf Crowdsourcing and research with human subjects}
    \item[] Question: For crowdsourcing experiments and research with human subjects, does the paper include the full text of instructions given to participants and screenshots, if applicable, as well as details about compensation (if any)? 
    \item[] Answer: \answerNA{} 
    \item[] Justification: This paper does not involve crowdsourcing or research with human subjects.

\item {\bf Institutional review board (IRB) approvals or equivalent for research with human subjects}
    \item[] Question: Does the paper describe potential risks incurred by study participants, whether such risks were disclosed to the subjects, and whether Institutional Review Board (IRB) approvals (or an equivalent approval/review based on the requirements of your country or institution) were obtained?
    \item[] Answer: \answerNA{} 
    \item[] Justification: This paper does not involve research with human subjects, so IRB approval is not applicable.

\item {\bf Declaration of LLM usage}
    \item[] Question: Does the paper describe the usage of LLMs if it is an important, original, or non-standard component of the core methods in this research? Note that if the LLM is used only for writing, editing, or formatting purposes and does \emph{not} impact the core methodology, scientific rigor, or originality of the research, declaration is not required.
    \item[] Answer: \answerYes{} 
    \item[] Justification: LLMs are the central research subjects of this work (Qwen3 and DeepSeek-R1-Distill families as students; Llama-Guard-3 and WildGuard as safety classifiers). Their usage is fully described in \Cref{sec:llm-usage}. The separate use of LLM assistants for writing and editing is also disclosed there. The student models generate training responses as part of the method; this is distinct from using assistants to produce the research itself.

\end{enumerate}

\end{document}